\newif\ifarxiv
\arxivtrue

\documentclass{article}

\ifarxiv
\usepackage{float}
\else
\usepackage{icml2023}
\fi

\usepackage[noend]{algorithmic}
\newcommand{\cZ}{\mathcal{Z}}
\newcommand{\cX}{\mathcal{X}}
\newcommand{\cY}{\mathcal{Y}}
\newcommand{\cH}{\mathcal{H}}
\newcommand{\cC}{\mathcal{C}}

\newcommand{\cG}{\mathcal{G}}
\newcommand{\cF}{\mathcal{F}}

\newcommand{\cD}{\mathcal{D}}

\newcommand{\includeimage}[1]{\includegraphics[trim={4cm 2.5cm 2.2cm 2.95cm}, clip, width=.15\linewidth,keepaspectratio,valign=m]{#1}}
\usepackage{graphicx}
\usepackage{collcell}
\usepackage{multirow}
\newcommand{\includepic}[1]{\includegraphics[trim={4cm 2.5cm 2.2cm 2.95cm}, clip, width=.17\linewidth,keepaspectratio,valign=m]{#1}}
\newcolumntype{j}{@{\hspace{1ex}}>{\collectcell\includeimage}c<{\endcollectcell}}
\newcolumntype{i}{@{\hspace{1ex}}>{\collectcell\includepic}c<{\endcollectcell}}

\newcolumntype{P}[1]{>{\centering\arraybackslash}p{#1}}
\usepackage[export]{adjustbox}

\usepackage[utf8]{inputenc}
\usepackage{tabularx}
\usepackage{amsmath,amssymb,amsthm,mathrsfs,pgf,tikz,caption,subcaption,mathtools,mathabx}
\usepackage[ruled,vlined,linesnumbered]{algorithm2e}
\usepackage{natbib}
\usepackage[utf8]{inputenc} 
\usepackage[T1]{fontenc}    
\usepackage{hyperref}       
\usepackage{url}            
\usepackage{booktabs}       
\usepackage{amsfonts}       
\usepackage{nicefrac}       
\usepackage{algorithmic}
\usepackage{todonotes}
\usepackage{graphicx}
\usepackage{natbib}
\usepackage{tikz}
\usepackage{bbm}
\usepackage{thmtools,thm-restate} 

\ifarxiv
\usepackage{fullpage, multicol}
\else
\fi

\ifarxiv
\title{Multicalibration as Boosting for Regression}
\author{Ira Globus-Harris \and Declan Harrison \and Michael Kearns \and Aaron Roth \and Jessica Sorrell}

\else
\fi 

\newtheorem{theorem}{Theorem}[section]
\newtheorem{corollary}[theorem]{Corollary}

\newtheorem{lemma}[theorem]{Lemma}
\newtheorem{definition}[theorem]{Definition}

\newtheorem{remark}[theorem]{Remark}

\newcommand{\R}{\mathbb{R}}

\newcommand{\dist}{\mathcal{D}}

\newcommand{\ind}{\mathbbm{1}}

\newcommand{\eps}{\varepsilon}
\newcommand{\E}{\mathop{\mathbb{E}}}
\newcommand{\ndim}{\textup{Ndim}}
\newcommand{\pdim}{\textup{Pdim}}
\newcommand{\round}{\textup{Round}}

\newcommand{\X}{\mathcal{X}}

\newcommand{\hyp}{\mathcal{H}}

\newcommand{\hboost}{\mathcal{H}_{\textup{boost}}}

\newcommand{\cL}{\mathcal{L}}

\begin{document}

\ifarxiv
\maketitle
\fi 

\begin{abstract}
We study the connection between multicalibration and boosting for squared error regression. First we prove a useful characterization of multicalibration in terms of a ``swap regret'' like condition on squared error. Using this characterization, we give an exceedingly simple algorithm that can be analyzed both as a boosting algorithm for regression and as a multicalibration algorithm for a class $\cH$ that makes use only of a standard squared error regression oracle for $\cH$. We give a weak learning assumption on $\cH$ that ensures convergence to Bayes optimality without the need to make any realizability assumptions --- giving us an agnostic boosting algorithm for regression. We then show that our weak learning assumption on $\cH$ is both necessary and sufficient for multicalibration with respect to $\cH$ to imply Bayes optimality. We also show that if $\cH$ satisfies our weak learning condition relative to another class $\cC$ then multicalibration with respect to $\cH$ implies multicalibration with respect to $\cC$. Finally we investigate the empirical performance of our algorithm experimentally using an open source implementation that we make available on GitHub\footnote{Our code repository can be found at \url{https://github.com/Declancharrison/Level-Set-Boosting}}. 
\end{abstract}

\section{Introduction}
We revisit the problem of boosting 
for regression, 
and develop a new agnostic regression boosting algorithm via a connection to multicalibration. In doing so, we shed additional light on multicalibration, a recent learning objective that has emerged from the algorithmic fairness literature \citep{hebert2018multicalibration}. In particular, we characterize multicalibration in terms of a ``swap-regret'' like condition, 
and use it to answer the question
``what  property must a collection of functions $\cH$ have so that multicalibration with respect to $\cH$ implies Bayes optimality?'', giving  a complete answer to problem asked by \cite{burhanpurkar2021scaffolding}. Using our swap-regret characterization, we derive an especially simple algorithm for learning a multicalibrated predictor for a class of functions $\cH$ by reduction to a standard squared-error regression algorithm for $\cH$. The same algorithm can also be analyzed as a boosting algorithm for squared error regression that makes calls to a weak learner for squared error regression on subsets of the original data distribution {\em without the need to relabel examples} (in contrast to Gradient Boosting as well as existing  multicalibration algorithms). This lets us specify a weak learning condition that is sufficient for convergence to the Bayes optimal predictor (even if the Bayes optimal predictor does not have zero error), avoiding the kinds of realizability assumptions that are implicit in analyses of boosting algorithms that converge to zero error. We conclude that ensuring multicalibration with respect to $\cH$ corresponds to boosting for squared error regression in which $\cH$ forms the set of weak learners. Finally we define a weak learning condition for $\cH$ relative to a constrained class of functions $\cC$ (rather than with respect to the Bayes optimal predictor). We show that multicalibration with respect to $\cH$ implies multicalibration with respect to $\cC$  if $\cH$ satisfies the weak learning condition with respect to $\cC$, which in turn implies accuracy at least that of the best function in $\cC$.

\paragraph{Multicalibration}
Consider a distribution $\cD \in \Delta \cZ$ defined over a domain $\cZ = \cX \times \mathbb{R}$ of feature vectors $x \in \cX$ paired with real valued labels $y$. Informally, a regression function $f:\cX\rightarrow \mathbb{R}$ is \emph{calibrated} if for every $v$ in the range of $f$, $\E_{(x,y) \sim \cD}[y | f(x) = v] = v$. In other words, $f(x)$ must be an unbiased estimator of $y$, even conditional on the value of its own prediction. Calibration on its own is a weak condition, because it only asks for $f$ to be unbiased on \emph{average} over all points $x$ such that $f(x) = v$. For example, the constant predictor that predicts $f(x) = \E_{(x,y) \sim \cD}[y]$ is calibrated. Thus calibration does not imply accuracy---a calibrated predictor need not make predictions with lower squared error than the best constant predictor. Calibration also does not imply that $f$ is equally representative of the label distribution on different subsets of the feature space $\cX$. For example, given a subset of the feature space $G \subseteq \cX$, even if $f$ is calibrated, it may be that $f$ is not calibrated on the conditional distribution conditional on $x \in G$---it might be e.g. that $\E[y | f(x)=v,x \in G] \gg v$, and $\E[y | f(x) = v, x\not\in G] \ll v$. To correct this last deficiency, \cite{hebert2018multicalibration} defined \emph{multi-calibration}, which is a condition parameterized by a subset of groups $G \subseteq \cX$ each defined by an indicator function $h:\cX\rightarrow \{0,1\}$ in some class $\cH$. It asks (informally) that for each such $h\in \cH$, and for each $v$ in the range of $f$, that $\E[h(x)(y-v) | f(x) = v] = 0$. Since $h$ is a binary indicator function for some set $G$, this is equivalent to asking for calibration not just marginally over $\cD$, but simultaneously for calibration over $\cD$ conditional on $x \in G$. \cite{kim2019multiaccuracy} and \cite{gopalan2022omnipredictors} generalize multicalibration beyond group indicator functions to arbitrary real valued functions $h :\cX\rightarrow \mathbb{R}$. Intuitively, as $\cH$ becomes a richer and richer set of functions, multicalibration becomes an increasingly stringent condition. But if $\cH$ consists of the indicator functions for e.g.~even a very large number of \emph{randomly selected} subsets $G \subseteq \cX$, then the constant predictor $f(x) = \E_{(x,y) \sim \cD}[y]$ will still be approximately multicalibrated with respect to $\cH$. What property of $\cH$  ensures that multicalibration with respect to $\cH$ implies that $f$ is a Bayes optimal regression function? This question was recently asked by \cite{burhanpurkar2021scaffolding} --- and we provide a necessary and sufficient condition. \\\\
\textbf{Boosting for Regression}
\ Boosting refers broadly to a collection of learning techniques that reduce the problem of ``strong learning'' (informally, finding an error optimal model) to a series of ``weak learning'' tasks (informally, finding a model that has only a small improvement over a trivial model)---See \cite{schapire2013boosting} for a textbook treatment. The vast majority of theoretical work on boosting studies the problem of binary classification, in which a weak learner is a learner that obtains classification error bounded below $1/2$. Several recent papers \cite{kim2019multiaccuracy,gopalan2022omnipredictors} have made connections between algorithms for guaranteeing multicalibration and boosting algorithms for binary classification. 

In this paper, we show a direct   connection between multicalibration and the much less well-studied problem of boosting for squared error regression \citep{friedman2001greedy,duffy2002boosting}. There is not a single established notion for what constitutes a weak learner in the regression setting (\cite{duffy2002boosting} introduce several different notions), and unlike boosting algorithms for classification problems which often work by calling a weak learner on a reweighting of the data distribution, existing algorithms for boosting for regression typically resort to calling a learning algorithm on \emph{relabelled} examples.  We give a boosting algorithm for regression that only requires calling a squared error regression learning algorithm on subsets of examples from the original distribution (without relabelling), which lets us formulate a weak learning condition that is sufficient to converge to the Bayes optimal predictor, without making the kinds of realizability assumptions implicit in the analysis of boosting algorithms that assume one can drive error to zero.

\subsection{Our Results}
We focus on classes of real valued functions $\cH$ that are closed under affine transformations --- i.e. classes such that if $f(x) \in \cH$, then for any pair of constants $a,b \in \mathbb{R}$, $(af(x) + b) \in \cH$ as well. Many natural classes of models satisfy this condition already (e.g. linear and polynomial  functions and regression trees), and any neural network architecture that does not already satisfy this condition can be made to satisfy it by adding two additional parameters ($a$ and $b$) while maintaining differentiability. Thus we view closure under affine transformations to be a weak assumption that is enforceable if necessary. 

First in Section \ref{sec:characterization} we prove the following characterization for multicalibration over $\cH$, for any class $\cH$ that is closed under affine transformations. Informally, we show that a model $f$ is multicalibrated with respect to $\cH$ if and only if, for every $v$ in the range of $f$:
\ifarxiv
$$\E_{(x,y) \sim \cD}[(f(x)-y)^2 | f(x) = v] \leq \min_{h \in \cH}\E_{(x,y) \sim \cD}[(h(x)-y)^2 | f(x) = v]$$
\fi 
(See Theorem \ref{thm:improve-multical-equivalent} for the formal statement). This is a ``swap regret''-like condition (as in \cite{foster1999regret} and \cite{blum2005external}), that states that $f$ must have lower squared error than any model $h \in \cH$, even conditional on its own prediction. Using this characterization, in Section \ref{sec:algorithm} we give an exceedingly simple algorithm for learning a multicalibrated predictor over $\cH$ given a squared error regression oracle for $\cH$. The algorithm simply repeats the following over $t$ rounds until convergence, maintaining a model $f:\cX\rightarrow \{0,1/m,2/m,\ldots,1\}$ with a discrete range with support over multiples of $1/m$ for some discretization factor $m$:
\begin{enumerate}
\item For each level set $v \in \{0,1/m,2/m,\ldots,1\}$, run a regression algorithm to find the $h^t_v \in \cH$ that minimizes squared error on the distribution $\cD|(f_{t-1}(x) = v)$, the distribution conditional on $f_{t-1}(x) = v$. 
\item Replace each level set $v$ of $f_{t-1}(x)$ with $h^t_v(x)$ to produce a new model $f_t$, and round its output to the discrete range $\{0,1/m,2/m,\ldots,1\}$
\end{enumerate}
Each iteration decreases the squared error of $f_t$, ensuring convergence, and our characterization of multicalibration ensures that we are multicalibrated with respect to $\cH$ at convergence. Compared to existing multicalibration algorithms (e.g. the split and merge algorithm of \cite{gopalan2022omnipredictors}), our algorithm is exceptionally simple and makes use of a standard squared-error regression oracle on subsets of the original distribution, rather than using a classification oracle or requiring example relabelling.

We can also view the same algorithm as a boosting algorithm for squared error regression. Suppose $\cH$ (or equivalently our weak learning algorithm) satisfies the following weak learning assumption: informally, that on any restriction of $\cD$ on which the Bayes optimal predictor is non-constant, there should be some $h \in \cH$ that obtains squared error better than that of the best constant predictor. Then our algorithm converges to the Bayes optimal predictor.  In Section \ref{sec:generalization} we give uniform convergence bounds which guarantee that the algorithm's accuracy and multicalibration guarantees generalize out of sample, with sample sizes that are linear in the pseudodimension of $\cH$. 

We then show in Section \ref{sec:accuracy} that in a strong sense this is the ``right'' weak learning assumption: Multicalibration with respect to $\cH$ implies Bayes optimality if and only if $\cH$ satisfies this weak learning condition. This gives a complete answer to the question of when multicalibration implies Bayes optimality. 

In \ifarxiv Section \ref{sec:constrained}\fi, we generalize our weak learning condition to a weak learning condition relative to a constrained class of functions $\cC$ (rather than relative to the Bayes optimal predictor), and show that if $\cH$ satisfies the weak learning condition relative to $\cC$, then multicalibration with respect to $\cH$ implies multicalibration with respect to $\cC$, and hence error that is competitive with the best model in $\cC$.

We give a fast, parallelizable implementation of our algorithm and in Section \ref{sec:experimental}  demonstrate its convergence to Bayes optimality on two-dimensional datasets useful for visualization, as well as evaluate the accuracy and calibration guarantees of our algorithm on real Census derived data using the Folktables package \cite{ding2021retiring}.

\subsection{Additional Related Work}
Calibration as a statistical objective dates back at least to \cite{dawid1982well}. \cite{foster1999regret} showed 
a tight connection between marginal calibration and internal (equivalently swap) regret. We extend this characterization to multicalibration.  Multicalibration was introduced by \cite{hebert2018multicalibration}, and variants of the original definition have been studied by a number of works  \citep{kim2019multiaccuracy,jung2021moment,gopalan2022omnipredictors,kim2022universal,rothuncertain}. We use the $\ell_2$ variant of multicalibration studied in \cite{rothuncertain}---but this definition implies all of the other variants of multicalibration up to a change in parameters. \cite{burhanpurkar2021scaffolding} first asked the question ``when does multicalibration with respect to $\cH$ imply accuracy'', and gave a sufficient condition: when $\cH$ contains (refinements of) the levelsets of the Bayes optimal regression function, together with techniques for attempting to find these. This can be viewed as a ``strong learning'' assumption, in contrast to our weak learning assumption on $\cH$.

Boosting for binary classification was introduced by \cite{schapire1990strength} and has since become a major topic of both theoretical and empirical study --- see \cite{schapire2013boosting} for a textbook overview. Both \cite{kim2019multiaccuracy} and \cite{gopalan2022omnipredictors} have drawn connections between algorithms for multicalibration and boosting for binary classification. In particular, \cite{gopalan2022omnipredictors} draw direct connections between their split-and-merge multicalibration algorithm and agnostic boosting algorithms of \cite{kalai2004learning,kanade2009potential,kalai2008agnostically}. Boosting for squared error regression is much less well studied. \cite{freund1997decision} give a variant of Adaboost (Adaboost.R) that reduces regression examples to infinite sets of classification examples, and requires a base regressor that optimizes a non-standard loss function. \cite{friedman2001greedy} introduced the popular gradient boosting method, which for squared error regression corresponds to iteratively fitting the residuals of the current model and then applying an additive update, but did not give a theoretical analysis. \cite{duffy2002boosting} give a theoretical analysis of several different boosting algorithms for squared error regression under several different weak learning assumptions. \ifarxiv Their algorithms require base regression algorithms that can be called (and guaranteed to succeed) on arbitrarily relabelled examples from the training distribution, and given their weak learning assumption, their analysis shows how to drive the error of the final model arbitrarily close to 0. Weak learning assumptions in this style implicitly make very strong realizabilty assumptions (that the Bayes error is close to 0), but because the weak learner is called on relabelled samples, it is difficult to enunciate a weak learning condition that is consistent with obtaining Bayes optimal error, but not better. The boosting algorithm we introduce only requires calling a standard regression algorithm on subsets of the examples from the training distribution, which makes it easy for us to define a weak learning condition that lets us drive error to the Bayes optimal rate without realizability assumptions --- thus our results can be viewed as giving an agnostic boosting algorithm for regression. \fi

\section{Preliminaries}
We study prediction tasks over a domain $\cZ = \cX\times \cY$. Here $\cX$ represents the \emph{feature} domain and $\cY$ represents the label domain. We focus on the bounded regression setting where $\cY=[0,1]$ (the scaling to $[0,1]$ is arbitrary). We write $\cD \in \Delta \cZ$ to denote a distribution over labelled examples, $\cD_\cX$ to denote the induced marginal distribution over features, and write $D \sim \cD^n$ to denote a dataset consisting of $n$ labelled examples sampled i.i.d. from $\cD$. We will be interested in the squared error of a model $f$ with respect to distribution $\dist$, $\E_{(x,y)\sim \dist}[(y - f(x))^2].$ We abuse notation and identify datasets $D = \{(x_1,y_1),\ldots,(x_n,y_n)\}$ with the empirical distribution over the examples they contain, and so we can write the empirical squared error over $D$: as
$\E_{(x,y)\sim D}[(y - f(x))^2]=\frac{1}{n}\sum_{i=1}^n(y_i - f(x_i))^2$.
When taking expectations over a distribution that is clear from context, we will frequently suppress notation indicating the relevant distribution for readability.

We write $R(f)$ to denote the range of a function $f$, and when $R(f)$ is finite, use $m$ to denote the cardinality of its range: $m = |R(f)|$. We are interested in finding models that are \textit{multicalibrated} with respect to a class of real valued functions $\cH$. We use an $\ell_2$ notion of multicalibration as used in \cite{rothuncertain}\ifarxiv:
\fi

\begin{definition}[Multicalibration]
\label{def:real-multicalibration}
Fix a distribution $\cD \in \Delta \cZ$ and a model $f:\cX\rightarrow [0,1]$ that maps onto a countable subset of its range. Let $\cH$ be an arbitrary collection of real valued functions $h:\cX \rightarrow \R$. We say that $f$ is $\alpha$-approximately multicalibrated with respect to $\cD$ and $\cH$ if for every $h \in \cH$:
\ifarxiv
$$K_2(f,h,\cD) = \sum_{v \in R(f)}\Pr_{(x,y) \sim \cD}[f(x) = v]\left(\E_{(x,y) \sim \cD}[h(x)(y - v) | f(x) = v]  \right)^2 \leq \alpha.$$
We say that $f$ is $\alpha$-approximately calibrated if:
$$K_2(f,\cD) = \sum_{v \in R(f)}\Pr_{(x,y) \sim \cD}[f(x) = v]\left(\E_{(x,y) \sim \cD}[(y - v) | f(x) = v]  \right)^2 \leq \alpha.$$
\fi
If $\alpha = 0$, then we simply say that a model is \emph{multicalibrated} or \emph{calibrated}. We will sometimes refer to $K_2(f,\cD)$ as the mean squared calibration error of a model $f$.
\end{definition}

\ifarxiv
\begin{remark}
When the functions $h(x)$ have binary range, we can view them as indicator functions for some subset of the data domain $S \subseteq \cX$, in which case multicalibration corresponds to asking for calibration conditional on membership in these subsets $S$. Allowing the functions $h$ to have real valued range is only a more general condition.  Our notion of approximate multicalibration takes a weighted average over the level sets $v$ of the predictor $f$, weighted by the probability that $f(x) = v$. This is necessary for any kind of out of sample generalization statement --- otherwise we could not even necessarily measure calibration error from a finite sample. Other work on multicalibration use related measures of multicalibration that we think of as $\ell_1$ or $\ell_\infty$ variants, that we can write as $K_1(f,h,\cD) = \sum_{v \in R(f)}\Pr_{(x,y) \sim \cD}[f(x) = v]\left|\E_{(x,y) \sim \cD}[h(x)(y - v) | f(x) = v]  \right|$ and $K_\infty(f,h,\cD) = \max_{v \in R(f)}\Pr_{(x,y) \sim \cD}[f(x) = v]\left(\E_{(x,y) \sim \cD}[h(x)(y - v) | f(x) = v]  \right)$. These notions are related to each other:
$K_2(f,h,\cD) \leq K_1(f,h,\cD) \leq \sqrt{K_2(f,h,\cD)}$ and $K_\infty(f,h,\cD) \leq K_1(f,h,\cD) \leq m K_\infty(f,h,\cD) $ \citep{rothuncertain}.
\end{remark}
\else
\fi 

We will characterize the relationship between multicalibration and Bayes optimality.
\begin{definition}[Bayes Optimal Predictor]
Let $f^*: \cX \rightarrow [0,1]$. We say that $f^*$ is the Bayes optimal predictor for $\cD$ if:
$$\E_{(x,y)\sim \dist}[(y - f^*(x))^2] \leq \min_{f:\cX\rightarrow [0,1]}[(y - f(x))^2]$$
The Bayes Optimal predictor satisfies:
$f^*(x) = \E_{(x',y)\sim \cD} \left[ y \vert x' = x \right].$
We say that a function $f:\cX\rightarrow [0,1]$ is $\gamma$-approximately Bayes optimal if 
\[
\E_{(x,y)\sim \dist}[(y - f(x))^2] \le \E_{(x,y)\sim \dist}[(y - f^*(x))^2] + \gamma.
\]
\end{definition}

Throughout this paper, we will denote the Bayes optimal predictor as $f^*$.

\section{A Characterization of Multicalibration}
\label{sec:characterization}
In this section we give a simple ``swap-regret'' like characterization of multicalibration for any class of functions $\cH$ that is closed under affine transformations:
\begin{definition}
A class of functions $\cH$ is closed under affine transformations if for every $a,b \in \R$, if $h(x) \in \cH$ then $h'(x) := a h(x) + b \in \cH$.
\end{definition}
As already discussed, closure under affine transformation is a mild assumption: it is already satisfied by many classes of functions $\cH$ like linear and polynomial functions and decision trees, and can be enforced for neural network architectures when it is not already satisfied by adding two additional parameters $a$ and $b$ without affecting our ability to optimize over the class.

 The first direction of our characterization states that if $f$ fails the multicalibration condition for some $h \in \cH$, then there is some other $h' \in \cH$ that improves over $f$ in terms of squared error, \emph{when restricted to a level set of $f$}. The second direction states the opposite: if $f$ is calibrated (but not necessarily multicalibrated), and if there is some level set of $f$ on which $h$ improves over $f$ in terms of squared error, then in fact $f$ must fail the multicalibration condition for  $h$. 

\begin{restatable}{theorem}{improveMulticalEquivalent}
\label{thm:improve-multical-equivalent}
Suppose $\cH$ is closed under affine transformation. Fix a model $f:\cX\rightarrow \R$ and a levelset $v \in R(f)$ of $f$. Then:
\begin{enumerate}
\item If there exists an $h \in \cH$ such that:
$$\E[h(x)(y-v)| f(x) = v] \geq \alpha, $$
for $\alpha > 0$, then there exists an $h' \in \cH$ such that:
\ifarxiv
$$\E[(f(x)-y)^2 - (h'(x) - y)^2 | f(x) = v] \geq \frac{\alpha^2}{\E[h(x)^2 | f(x) =v]},$$
\fi 
\item If $f$ is calibrated  and there exists an $h \in \cH$ such that 
$$\E[(f(x)-y)^2 - (h(x) - y)^2 | f(x) = v] \geq \alpha,$$ then:
$$\E[h(x)(y-v) | f(x) = v] \geq \frac{\alpha}{2}.$$
\end{enumerate}
\end{restatable}

\ifarxiv
\begin{proof}
We prove each direction in turn.
\begin{lemma}
\label{lem:multicalToImprovement}
Fix a  model $f:\cX\rightarrow \R$. Suppose for some $v \in R(f)$ there is an $h \in \cH$ such that:
$$\E[h(x)(y-v)| f(x) = v] \geq \alpha$$
Let $h' = v + \eta h(x)$ for $\eta = \frac{\alpha}{\E[h(x)^2 | f(x) =v]}$. Then:
$$\E[(f(x)-y)^2 - (h'(x) - y)^2 | f(x) = v] \geq \frac{\alpha^2}{\E[h(x)^2 | f(x) =v]}$$
\end{lemma}
\begin{proof}
We calculate:
\begin{eqnarray*}
&& \E[(f(x)-y)^2 - (h'(x) - y)^2 | f(x) = v] \\
&=& \E[(v-y)^2 - (v+\eta h(x) - y)^2 | f(x) = v] \\
\ifarxiv
&=& \E[v^2 - 2vy + y^2 - (v + \eta h(x))^2 + 2y(v+\eta h(x)) - y^2 | f(x) = v] \\
\else
&=& \E[v^2 - 2vy + y^2 - (v + \eta h(x))^2 \\
&& + 2y(v+\eta h(x)) - y^2 | f(x) = v] \\
\fi 
&=& \E[2y\eta h(x) - 2v \eta h(x) - \eta^2 h(x)^2 | f(x) = v] \\
&=& \E[2\eta h(x)(y-v) - \eta^2 h(x)^2 | f(x) = v] \\
&\geq& 2\eta \alpha - \eta^2 \E[h(x)^2 | f(x) = v] \\
&=& \frac{\alpha^2}{\E[h(x)^2|f(x) = v]}
\end{eqnarray*}
Where the last line follows from the definition of $\eta$. 
\end{proof}
The first direction of Theorem \ref{thm:improve-multical-equivalent} follows from Lemma \ref{lem:multicalToImprovement}, and the observation that since $\cH$ is closed under affine transformations, the function $h'$ defined in the statement of Lemma \ref{lem:multicalToImprovement} is in $\cH$. Now for the second direction.
\begin{lemma}
\label{lem:improvementToMultical}
Fix a model  $f:\cX\rightarrow \R$. Suppose for some $v \in R(f)$ there is an $h \in \cH$ such that:
$$\E[(\bar{y}_v-y)^2 - (h(x) - y)^2 | f(x) = v] \geq \alpha,$$
where $\bar{y}_v = \E[y \mid f(x) = v]$.
Then it must be that:
$$\E[h(x)(y-\bar{y}_v) | f(x) = v] \geq \frac{\alpha}{2}$$
\end{lemma}
\begin{proof}
We calculate: 

\begin{eqnarray*}
&&\E_{(x,y) \sim \cD}[h(x)(y-\bar{y}_v) | f(x) = v] \\ &=& \E_{(x,y) \sim \cD}[h(x)y | f(x) = v] - \bar{y}_v \E_{(x,y) \sim \cD}[h(x) | f(x) = v] \\
&=& \frac{1}{2}\left(2\E_{(x,y) \sim \cD}[h(x)y | f(x) = v] - 2  \bar{y}_v \E_{(x,y) \sim \cD}[h(x) | f(x) = v] \right) \\
\ifarxiv
&\geq&  \frac{1}{2}\left(2\E_{(x,y) \sim \cD}[h(x)y | f(x) = v] - 2  \bar{y}_v \E_{(x,y) \sim \cD}[h(x) | f(x) = v] - \E_{(x,y) \sim \cD}[(h(x) - \bar{y}_v)^2 | f(x) = v] \right) \\
\else 
&\geq&  \frac{1}{2}\bigg(2\E_{(x,y) \sim \cD}[h(x)y | f(x) = v] - 2  \bar{y}_v \E_{(x,y) \sim \cD}[h(x) | f(x) = v] \\
&& - \E_{(x,y) \sim \cD}[(h(x) - \bar{y}_v)^2 | f(x) = v] \bigg) \\
\fi 
&=& \frac{1}{2} \left( \E_{(x,y) \sim \cD}[2h(x) y - h(x)^2 - \bar{y}_v^2 | f(x) = v] \right) \\ 
&=& \frac{1}{2} \left( \E_{(x,y) \sim \cD}[2h(x) y - h(x)^2 - 2\bar{y}_vy +  \bar{y}_v^2 | f(x) = v] \right) \\
&=& \frac{1}{2}\left( \E_{(x,y) \sim \cD}[(\bar{y}_v - y)^2 - (h(x)- y)^2 | f(x) = v]\right) \\
&\geq& \frac{\alpha}{2}
\end{eqnarray*}
where the 3rd to last line follows from adding and subtracting $\bar{y}_v^2$.
\end{proof}
For any calibrated $f$ it follows that $v = \E[y \mid f(x) = v] = \bar{y}_v$, and so for calibrated $f$ we have that if 
$$\E[(v-y)^2 - (h(x) - y)^2 | f(x) = v] \geq \alpha,$$
then:
$$\E[h(x)(y-v) | f(x) = v] \geq \frac{\alpha}{2}.$$
\end{proof}
\fi 

\section{An Algorithm (For Multicalibration And Regression Boosting)}
\label{sec:algorithm}
We now give a single algorithm, and then show how to analyze it both as an algorithm for obtaining a multicalibrated predictor $f$, and as a boosting algorithm for squared error regression. 

 Let $m \in \mathbb{N}^+$ be a discretization term, and let $[1/m] := \{0, \frac{1}{m}, \ldots, \frac{m-1}{m}, 1\}$ denote the set of points in $[0,1]$ that are multiples of $1/m$. We will learn a model $f$ whose range is $[1/m]$, which we will enforce by \emph{rounding} its outputs to this range as necessary using the following operation:

\begin{definition}[$\textrm{Round}(f; m)$] Let $\cF$ be the family of all functions $f: \cX \rightarrow \R$. Let $\textup{Round}: \cF \times \mathbb{N}^+ \rightarrow \cF$ be a function such that $\textup{Round}(f; m)$ outputs $\tilde h(x) = \min_{v \in [1/m]} \vert h(x) - v \vert$.
\end{definition}
Unlike other algorithms for multicalibration which make use of \emph{agnostic learning} oracles for binary classification, our algorithm  makes use of an algorithm for solving squared-error regression problems over $\cH$:
\begin{definition}
$A_\cH$ is a squared error regression oracle for a class of real valued functions $\cH$ if for every $\cD \in \Delta \cZ$, $A_\cH(\cD)$ outputs a function $h \in \cH$ such that 
\ifarxiv
\[h \in \arg\min_{h' \in \cH}\E_{(x,y) \sim \cD}[(h'(x) - y)^2].\]
\fi 
\end{definition}
For example, if $\cH$ is the set of all linear functions, then $A_\cH$ simply solves a linear regression problem (which has a closed form solution). Algorithm \ref{alg:regression-multicalibrator} (LSBoost\footnote{LSBoost can be taken to stand for either ``Level Set Boost" or ``Least Squares Boost'', at the reader's discretion.})repeats the following operation until it no longer decreases overall squared error: it runs squared error regression on each of the level-sets of $f_t$, and then replaces those levelsets with the solutions to the regression problems, and rounds the output to $[1/m]$. 

We will now analyze the algorithm first as a multicalibration algorithm, and then as a boosting algorithm. For simplicity, in this section we will analyze the algorithm as if it is given direct access to the distribution $\cD$. In practice, the algorithm will be run on the empirical distribution over a dataset $D \sim \cD^n$, and the multicalibration guarantees proven in this section will hold for this empirical distribution. In \ifarxiv Section \fi \ref{sec:generalization} we prove generalization theorems, which allow us to translate our in-sample error and multicalibration guarantees over $D$ to out-of-sample guarantees over $\cD$.


\begin{algorithm}[H]
\begin{algorithmic}
\STATE Let $m = \frac{2B}{\alpha}$.
\STATE \textbf{Let} $f_0 = \textrm{Round}(f ; m)$, $\textrm{err}_0 = \E_{(x,y) \sim \cD}[(f_0(x) - y)^2]$,
\STATE $\textrm{err}_{-1} = \infty$ and $t = 0$.

\WHILE{$\left(\textrm{err}_{t-1} - \textrm{err}_t\right) \geq \frac{\alpha}{2B}$}
   \FOR{each $v \in [1/m]$}
    \STATE \textbf{Let} $\cD_v^{t+1} = \cD | (f_t(x) = v)$.
    \STATE \textbf{Let} $h_v^{t+1} = A_\cH(\cD_v^{t+1})$.
  \ENDFOR
  \ifarxiv
  \STATE \textbf{Let:} 
  $$\tilde f_{t+1}(x) = \sum_{v \in [1/m]} \mathbbm{1}[f_t(x) = v]\cdot h_v^{t+1}(x) \ \ \ f_{t+1} = \textrm{Round}(\tilde f_{t+1}, m)$$
  \fi 
  \STATE \textbf{Let} $\textrm{err}_{t+1} =\E_{(x,y) \sim \cD}[(f_{t+1}(x) - y)^2]$ and $t = t+1$.
\ENDWHILE
\STATE \textbf{Output} $f_{t-1}$. 
\end{algorithmic}
\caption{LSBoost($f,\alpha,A_{\cH},\cD,B$)}
\label{alg:regression-multicalibrator}
\end{algorithm}

\subsection{Analysis as a Multicalibration Algorithm}
\begin{restatable}{theorem}{AlgAnalysis}
\label{thm:alg-analysis}
Fix any distribution $\cD \in \Delta \cZ$, any model $f:\cX \rightarrow [0,1]$, any $\alpha < 1$, any class of real valued functions $\cH$ that is closed under affine transformations, and a squared error regression oracle $A_\cH$ for $\cH$. 
For any bound $B > 0$ let:
$$\cH_B = \{h \in \cH : \max_{x\in \cX} h(x)^2 \leq B\}$$
be the set of functions in $h$ with squared magnitude bounded by $B$. 
Then LSBoost$(f,\alpha,A_{\cH},\cD,B)$ (Algorithm \ref{alg:regression-multicalibrator}) halts after at most $T \leq \frac{2B}{\alpha}$ many iterations and outputs a model $f_{T-1}$ such that $f_{T-1}$ is $\alpha$-approximately multicalibrated with respect to $\cD$ and $\cH_B$.
\end{restatable}

\ifarxiv
\begin{remark}
\label{rem:hbound}
Note the form of this theorem --- we do not promise multicalibration at approximation parameter $\alpha$ for all of $\cH$, but only for $\cH_B$ --- i.e. those functions in $\cH$ satisfying a bound on their squared value. This is necessary, since $\cH$ is closed under affine transformations. To see this, note that if $\E[h(x)(y-v)] \geq \alpha$, then it must be that $\E[c\cdot h(x)(y-v)] \geq c\cdot \alpha$. Since $h'(x) = c h(x)$ is also in $\cH$ by assumption, approximate multicalibration bounds must always also be paired with a bound on the norm of the functions for which we promise those bounds. 
\end{remark}

\begin{proof}
Since $f_0$ takes values in $[0,1]$ and $y \in [0,1]$, we have $\textrm{err}_0 \leq 1$, and by definition $\textrm{err}_T \geq 0$ for all $T$. By construction, if the algorithm has not halted at round $t$ it must be that $\textrm{err}_t \leq \textrm{err}_{t-1} - \frac{\alpha}{2B}$, and so the algorithm must halt after at most $T \leq \frac{2B}{\alpha}$ many iterations to avoid a contradiction.

It remains to show that when the algorithm halts at round $T$, the model $f_{T-1}$ that it outputs is $\alpha$-approximately multi-calibrated with respect to $\cD$ and $\cH_B$. We will show that if this is not the case, then $\textrm{err}_{T-1} - \textrm{err}_{T} > \frac{\alpha}{2B}$, which will be a contradiction to the halting criterion of the algorithm. 

Suppose that $f_{T-1}$ is not $\alpha$-approximately multicalibrated with respect to $\cD$ and $\cH_B$. This means there must be some $h \in \cH_B$ such that:
\ifarxiv
$$\sum_{v \in [1/m]}\Pr_{(x,y) \sim \cD}[f_{T-1}(x) = v]\left(\E_{(x,y) \sim \cD}[h(x)(y-v) | f_{T-1}(x) = v] \right)^2 > \alpha$$
\else 
$\sum_{v \in [1/m]}\Pr_{(x,y) \sim \cD}[f_{T-1}(x) = v]\left(\E_{(x,y) \sim \cD}[h(x)(y-v) | f_{T-1}(x) = v] \right)^2 > \alpha$.
\fi 

For each $v \in [1/m]$ define 
\ifarxiv 
$$\alpha_v =\Pr_{(x,y) \sim \cD}[f_{T-1}(x) = v]\left(\E_{(x,y) \sim \cD}[h(x)(y-v) | f_{T-1}(x) = v] \right)^2$$
\else
$\alpha_v =\Pr_{(x,y) \sim \cD}[f_{T-1}(x) = v]\left(\E_{(x,y) \sim \cD}[h(x)(y-v) | f_{T-1}(x) = v] \right)^2$
\fi 
So we have $\sum_{v \in [1/m]} \alpha_v > \alpha$. 

Applying the 1st part of Theorem \ref{thm:improve-multical-equivalent} we learn that for each $v$, there must be some $h_v \in \cH$ such that:
\ifarxiv
\begin{eqnarray*}
\E[(f_{T-1}(x)-y)^2 - (h_v(x) - y)^2 | f_{T-1}(x) = v] &>& \frac{1}{\E[h(x)^2 | f_{T-1}(x) =v]}\cdot \frac{\alpha_v}{\Pr_{(x,y) \sim \cD}[f_{T-1}(x) = v]} \\
&\geq& \frac{1}{B}\frac{\alpha_v}{\Pr_{(x,y) \sim \cD}[f_{T-1}(x) = v]}
\end{eqnarray*}
\else
\begin{eqnarray*}
&& \E[(f_{T-1}(x)-y)^2 - (h_v(x) - y)^2 | f_{T-1}(x) = v] \\
&>& \frac{1}{\E[h(x)^2 | f_{T-1}(x) =v]}\cdot 
  \frac{\alpha_v}{\Pr_{(x,y) \sim \cD}[f_{T-1}(x) = v]} \\
&\geq& \frac{1}{B}\frac{\alpha_v}{\Pr_{(x,y) \sim \cD}[f_{T-1}(x) = v]}
\end{eqnarray*}
\fi 
where the last inequality follows from the fact that $h \in \cH_B$
Now we can compute:
\begin{eqnarray*}
&& \E_{(x,y) \sim \cD}[(f_{T-1}(x) - y)^2 - (\tilde f_{T}(x) - y)^2] \\
&=& \sum_{v \in [1/m]} \Pr_{(x,y) \sim \cD}[f_{T-1}(x)=v]\E_{(x,y) \sim \cD}[(f_{T-1}(x) - y)^2 - (\tilde f_{T}(x) - y)^2 | f_{T-1}(x)=v] \\
&=& \sum_{v \in [1/m]} \Pr_{(x,y) \sim \cD}[f_{T-1}(x)=v]\E_{(x,y) \sim \cD}[(f_{T-1}(x) - y)^2 - (h_v^T(x) - y)^2 | f_{T-1}(x)=v] \\
&\geq&  \sum_{v \in [1/m]} \Pr_{(x,y) \sim \cD}[f_{T-1}(x)=v]\E_{(x,y) \sim \cD}[(f_{T-1}(x) - y)^2 - (h_v(x) - y)^2 | f_{T-1}(x)=v]  \\
&\geq&  \sum_{v \in [1/m]} \frac{\alpha_v}{B} \\
&>& \frac{\alpha}{B}
\end{eqnarray*}
Here the third line follows from the definition of $\tilde f_T$ and the fourth line follows from the fact $h_v \in \cH$ and that $h_v^T$ minimizes squared error on $\cD^T_v$ amongst all $h \in \cH$.

Finally we calculate:
\begin{eqnarray*}
&& \textrm{err}_{T-1} - \textrm{err}_{T} \\
&=& \E_{(x,y) \sim \cD}[(f_{T-1}(x) - y)^2 - (f_{T}(x) - y)^2] \\
&=& \E_{(x,y) \sim \cD}[(f_{T-1}(x) - y)^2 - (\tilde f_{T}(x) - y)^2]  + \E_{(x,y) \sim \cD}[(\tilde f_{T}(x) - y)^2 - (f_{T}(x) - y)^2] \\
&>& \frac{\alpha}{B} + \E_{(x,y) \sim \cD}[(\tilde f_{T}(x) - y)^2 - (f_{T}(x) - y)^2] \\
&>& \frac{\alpha}{B} - \frac{1}{m} \\
&\geq& \frac{\alpha}{2B}
\end{eqnarray*}
where the last equality follows from the fact that $m \geq \frac{2B}{\alpha}$.

The 2nd inequality follows from the fact that for every pair $(x,y)$: 
$$(\tilde f_{T}(x) - y)^2 - (f_{T}(x) - y)^2 \geq -\frac{1}{m}$$
To see this we consider two cases. Since $y \in [0,1]$, if $\tilde f_T(x) > 1$ or $\tilde f_T(x) < 0$ then the Round operation decreases squared error and we have  $(\tilde f_{T}(x) - y)^2 - (f_{T}(x) - y)^2 \geq 0$. In the remaining case we have $f_T(x) \in [0,1]$ and $\Delta =  \tilde f_T(x) -  f_T(x)$ is such that $|\Delta| \leq \frac{1}{2m}$. In this case we can compute:  
\begin{eqnarray*}
(\tilde f_{T}(x) - y)^2 - (f_{T}(x) - y)^2 
&=& (f_T(x) + \Delta - y)^2 - (f_{T}(x) - y)^2 \\ 
&=& 2\Delta (f(x) - y) + \Delta^2 \\
&\geq& -2|\Delta| + \Delta^2 \\
&\geq& -\frac{1}{m} 
\end{eqnarray*}
\end{proof}
\fi 

\subsection{Analysis as a Boosting Algorithm}
We now analyze the same algorithm (Algorithm \ref{alg:regression-multicalibrator}) as a boosting algorithm designed to boost a ``weak learning'' algorithm $A_\cH$ to a strong learning algorithm. Often in the boosting literature, a ``strong learning'' algorithm is one that can obtain accuracy arbitrarily close to perfect, which is only possible under strong realizability assumptions. In this paper, by ``strong learning'', we mean that Algorithm \ref{alg:regression-multicalibrator} should output a model that is close to Bayes optimal, which is a goal we can enunciate for any distribution $\cD$ without needing to make realizability assumptions. (Observe that if the Bayes optimal predictor has zero error, then our meaning of strong learning corresponds to the standard meaning, so our analysis is only more general). 

We now turn to our definition of weak learning. Intuitively, a weak learning algorithm should return a hypothesis that makes predictions that are slightly better than trivial whenever doing so is possible. We take ``trivial'' predictions to be those of the best \emph{constant} predictor as measured by squared error --- i.e. the squared error obtained by simply returning the label mean. 
A ``weak learning'' algorithm for us can be run on any restriction of the data distribution $\cD$ to a subset $S \subseteq \cX$, and must return a hypothesis with squared error slightly better than the squared error of the best constant prediction, whenever the Bayes optimal predictor $f^*$ has squared error slightly better than a constant predictor; on restrictions for which the Bayes optimal predictor also does not improve over constant prediction, our weak learning algorithm is not required to do better either. 

Traditionally, ``weak learning'' assumptions do not distinguish between the optimization ability of the algorithm and the representation ability of the hypothesis class it optimizes over. Since we have defined a squared error regression oracle $A_\cH$ as exactly optimizing the squared error over some class $\cH$, we will state our weak learning assumption as an assumption on the representation ability of $\cH$---but this is not important for our analysis here. To prove Theorem \ref{thm:boosting} we could equally well assume that $A_\cH$ returns a hypothesis $h$ that improves over a constant predictor whenever one exists, without assuming that $h$ optimizes squared error over all of $\cH$. 
\begin{definition}[Weak Learning Assumption]
\label{def:weaklearner-quant}
Fix a distribution $\cD \in \Delta \cZ$ and a class of functions $\cH$. Let $f^*(x) = \E_{y \sim \cD(x)}[y]$ denote the true conditional label expectation conditional on $x$. We say that $\cH$ satisfies the $\gamma$-\emph{weak learning} condition relative to $\cD$ if for every $S \subseteq \cX$ with $\Pr_{x \sim \cD_\cX}[x \in S] > 0$, if:
$$\E[(f^*(x) - y)^2 | x \in S] < \min_{c \in \R} \E[(c - y)^2 | x \in S] - \gamma $$
then there exists an $h \in \cH$ such that:
$$\E[(h(x) - y)^2 | x \in S] < \min_{c \in \R} \E[(c - y)^2 | x \in S] - \gamma $$
When $\gamma = 0$ we simply say that $\cH$ satisfies the weak learning condition relative to $\cD$.
\end{definition}
Observe why our weak learning assumption is ``weak'': the Bayes optimal predictor $f^*$ may improve arbitrarily over the best constant predictor on some set $S$ in terms of squared error, but in this case we only require of $\cH$  that it include a hypothesis that improves by some $\gamma$ which might be very small. 

Since the $\gamma$-weak learning condition does not make any requirements on $\cH$ on sets for which $f^*(x)$ improves over a constant predictor by less than $\gamma$, the best we can hope to prove under this assumption is $\gamma$-approximate Bayes optimality, which is what we do next.

\begin{restatable}{theorem}{thmBoosting}
\label{thm:boosting}
Fix any distribution $\cD \in \Delta \cZ$, any model $f:\cX\rightarrow [0,1]$, any $\gamma >0$,  any class of real valued functions $\cH$ that satisfies the $\gamma$-weak learning condition relative to $\cD$, and a squared error regression oracle $A_\cH$ for $\cH$. Let $\alpha = \gamma$ and $B = 1/\gamma$ (or any pair such that $\alpha/B = \gamma^2$). Then LSBoost$(f,\alpha,A_\cH,\cD,B)$ halts after at most $T \leq \frac{2}{\gamma^2}$ many iterations and outputs a model $f_{T-1}$ such that $f_{T-1}$ is $2\gamma$-approximately Bayes optimal over $\cD$:
$$\E_{(x,y) \sim \cD}[(f_{T-1}(x) - y)^2] \leq \E_{(x,y) \sim \cD}[(f^*(x) - y)^2] + 2\gamma$$
where $f^*(x) = \E_{(x,y) \sim \cD}[y]$ is the function that minimizes squared error over $\cD$. 
\end{restatable}
\ifarxiv
\begin{proof}
At each round $t$ before the algorithm halts, we have by construction that $\textrm{err}_{t} \leq \textrm{err}_{t-1} - \frac{\alpha}{2 B}$, and since the squared error of $f_0$ is at most $1$, and squared error is non-negative, we must have $T \leq \frac{2B}{\alpha} = \frac{2}{\gamma^2}$.

Now suppose the algorithm halts at round $T$ and outputs $f_{T-1}$. It must be that $\textrm{err}_T > \textrm{err}_{T-1} - \frac{\gamma^2}{2}$. Suppose also that $f_{T-1}$ is not $2\gamma$-approximately Bayes optimal:
$$\E_{(x,y) \sim \cD}[(f_{T-1}(x) - y)^2 - (f^*(x) - y)^2] >  2\gamma$$
We can write this condition as:
$$\sum_{v \in [1/m]}\Pr[f_{T-1}(x) = v]\cdot \E_{(x,y) \sim \cD}[(f_{T-1}(x) - y)^2 - (f^*(x) - y)^2 | f_{T-1}(x) = v] >  2\gamma$$
Define the set:
$$S = \{v \in [1/m] : \E_{(x,y) \sim \cD}[(f_{T-1}(x) - y)^2 - (f^*(x) - y)^2 | f_{T-1}(x) = v] \geq \gamma\}$$
to denote the set of values $v$ in the range of $f_{T-1}$ such that conditional on $f_{T-1}(x) = v$, $f_{T-1}$ is at least $\gamma$-sub-optimal. Since we have both $y \in [0,1]$ and $f_{T-1}(x) \in [0,1]$, for every $v$ we must have that $\E[(f_{T-1}(x) - y)^2 - (f^*(x) - y)^2 | f_{T-1}(x) = v] \leq 1$. Therefore we can bound:
\begin{eqnarray*}
2\gamma &<& \sum_{v \in [1/m]}\Pr[f_{T-1}(x) = v]\cdot \E_{(x,y) \sim \cD}[(f_{T-1}(x) - y)^2 - (f^*(x) - y)^2 | f_{T-1}(x) = v] \\
&\leq& \Pr_{(x,y) \sim \cD}[x \in S] + (1-\Pr_{(x,y) \sim \cD}[x \in S]) \gamma
\end{eqnarray*}
Solving we learn that:
$$\Pr_{(x,y) \sim \cD}[x \in S] \geq \frac{2\gamma-\gamma}{(1-\gamma)} \geq 2\gamma - \gamma = \gamma$$

Now observe that by the fact that $\cH$ is assumed to satisfy the $\gamma$-weak learning assumption with respect to $\cD$, at the final round $T$ of the algorithm, for every $v \in S$ we have that $h_v^T$ satisfies:
$$\E_{(x,y) \sim \cD}[(f_{T-1}(x) - y)^2 - (h_v^T(x) - y)^2 | f_{T-1}(x) = v] \geq \gamma$$
Let $\tilde {\textrm{err}}_T = \E_{(x,y) \sim \cD}[(\tilde f_T(x) - y)^2]$
Therefore we have:
\begin{eqnarray*}
\textrm{err}_{T-1}-\tilde{\textrm{err}}_T &=& \sum_{v \in [1/m]}\Pr_{(x,y) \sim \cD}[f_{T-1}(x) = v]\E_{(x,y) \sim \cD}[(f_{T-1}(x) - y)^2 - (h_v^T(x) - y)^2 | f_{T-1}(x) = v] \\
&\geq& \Pr_{(x,y) \sim \cD}[f_{T-1}(x) \in S] \gamma \\
&\geq& \gamma^2
\end{eqnarray*}
We recall that $|\tilde{\textrm{err}}_T - \textrm{err}_T| \leq 1/m = \frac{\gamma^2}{2}$ and so we can conclude that 
$$\textrm{err}_{T-1}-\textrm{err}_T \geq \frac{\gamma^2}{2}$$
which contradicts the fact that the algorithm halted at round $T$, completing the proof. 
\end{proof}
\fi

\section{When Multicalibration Implies Accuracy}
\label{sec:accuracy}
We analyzed the same algorithm (Algorithm \ref{alg:regression-multicalibrator}) as both an algorithm for obtaining multicalibration with respect to $\cH$, and, when $\cH$ satisfied the weak learning condition given in Definition \ref{def:weaklearner-quant}, as a boosting algorithm that converges to the Bayes optimal model. In this section we show that this is no coincidence: multicalibration with respect to $\cH$ implies Bayes optimality if and only if $\cH$ satisfies the weak learning condition from Definition \ref{def:weaklearner-quant},

First we define what we mean when we say that multicalibration with respect to $\cH$ implies  Bayes optimality. Note that the Bayes optimal model $f^*(x)$ is multicalibrated with respect to any set of functions, so it is not enough to require that there \emph{exist} Bayes optimal functions $f$ that are multicalibrated with respect to $\cH$. Instead, we have to require that \emph{every} function that is multicalibrated with respect to $\cH$ is Bayes optimal:
\begin{definition}
Fix a distribution $\cD \in \Delta \cZ$. We say that multicalibration with respect to $\cH$ implies Bayes optimality over $\cD$ if for every $f:\cX \rightarrow \mathbb{R}$ that is multicalibrated with respect to $\cD$ and $\cH$, we have:
$$\E_{(x,y) \sim \cD}[(f(x) - y)^2] = \E_{(x,y) \sim \cD}[(f^*(x) - y)^2]$$
Where $f^*(x) = \E_{y \sim \cD(x)}[y]$ is the function that has minimum squared error over the set of all functions. 
\end{definition}

Recall that when the weak learning parameter $\gamma$ in Definition \ref{def:weaklearner-quant} is set to $0$, we simply call it the ``weak learning condition'' relative to $\cD$. We first state and prove our characterization for the exact case when $\gamma = 0$, because it leads to an exceptionally simple statement. We subsequently extend this characterization to relate approximate Bayes optimality and approximate multicalibration under quantitative weakenings of the weak learning condition. 

\begin{restatable}{theorem}{thmWeaklearnercharacterize}
\label{thm:weaklearnercharacterize}
Fix a distribution $\cD \in \Delta \cZ$. Let $\cH$ be a class of functions that is closed under affine transformation.  Multicalibration with respect to $\cH$ implies Bayes optimality over $\cD$ if and only if $\cH$ satisfies the weak learning condition relative to $\cD$. 
\end{restatable}

\ifarxiv
\begin{proof}
To avoid measurability issues we assume that models $f$ have a countable range (which is true in particular whenever $\cX$ is countable).

First we show that if $\cH$ satisfies the weak learning condition relative to $\cD$, then multicalibration with respect to $\cH$ implies Bayes optimality over $\cD$. Suppose not. Then there exists a function $f$ that is multicalibrated with respect to $\cD$ and $\cH$, but is such that:
$$\E_{(x,y) \sim \cD}[(f(x) - y)^2] > \E_{(x,y) \sim \cD}[(f^*(x) - y)^2]$$

By linearity of expectation we have:
$$\sum_{v \in R(f)}\Pr[f(x) = v]\cdot \E_{(x,y) \sim \cD}[(f(x) - y)^2 - (f^*(x) - y)^2 | f(x) = v] > 0$$

In particular  there must be some $v \in R(f)$ with $\Pr_{x \sim \cD_\cX}[f(x) = v] > 0$ such that:
$$\E_{(x,y) \sim \cD}[(f(x) - y)^2 | f(x) = v] > \E_{(x,y) \sim \cD}[(f^*(x) - y)^2 | f(x) = v]$$

Let $S = \{x : f(x) = v\}$. Observe that if $\cH$ is closed under affine transformation, the constant function $h(x) = 1$ is in $\cH$, and hence multicalibration with respect to $\cH$ implies calibration. Since $f$ is calibrated, we know that: 
$$\E_{(x,y) \sim \cD}[(v - y)^2 | x \in S] = \min_{c \in \R} \E_{(x,y) \sim \cD}[(c - y)^2 | x \in S]$$ Thus by the weak learning assumption there must exist some $h \in \cH$ such that:
$$\E[(v-y)^2 - (h(x) - y)^2 | x \in S] = \E[(f(x)-y)^2 - (h(x) - y)^2 | f(x) = v] > 0$$

By Theorem \ref{thm:improve-multical-equivalent}, there must therefore exist some $h' \in \cH$ such that:
$$\E_{(x,y) \sim \cD}[h'(x)(y-v) | f(x) = v] > 0$$ implying that $f$ is \emph{not} multicalibrated with respect to $\cD$ and $\cH$, a contradiction. 

In the reverse direction, we show that for any $\cH$ that does \emph{not} satisfy the weak learning condition with respect to $\cD$, then multicalibration with respect to $\cH$ and $\cD$ does not imply Bayes optimality over $\cD$. In particular, we exhibit a function $f$ such that $f$ is multicalibrated with respect to $\cH$ and $\cD$, but such that:
$$\E_{(x,y) \sim \cD}[(f(x)-y)^2] > \E_{(x,y) \sim \cD}[(f^*(x)-y)^2]$$

Since $\cH$ does not satisfy the weak learning assumption over $\cD$, there must exist some set $S \subseteq \cX$ with $\Pr[x \in S] > 0$ such that 
$$\E_{(x,y) \sim \cD}[(f^*(x) - y)^2 | x \in S] < \min_{c \in \R} \E_{(x,y) \sim \cD}[(c - y)^2 | x \in S] $$
but for every $h \in \cH$:
$$\E_{(x,y) \sim \cD}[(h(x) - y)^2 | x \in S] \geq  \min_{c \in \R} \E_{(x,y) \sim \cD}[(c - y)^2 | x \in S] $$.

Let $c(S) = \E_{(x,y) \sim \cD}[y | x \in S]$. We define $f(x)$ as follows:
$$f(x) = \begin{cases}
  f^*(x)  & x \not \in S \\
  c(S) & x \in S
\end{cases}$$
We can calculate that:
\begin{eqnarray*}
&& \E_{(x,y) \sim \cD}[(f(x)-y)^2] \\ &=& \Pr_{(x,y) \sim \cD}[x \in S] \E_{(x,y) \sim \cD}[(c(S)-y)^2 | x \in S] + \Pr_{(x,y) \sim \cD}[x \not\in S] \E_{(x,y) \sim \cD}[(f^*(x)-y)^2 | x \not \in S] \\
&>& \Pr_{(x,y) \sim \cD}[x \in S] \E_{(x,y) \sim \cD}[(f^*(x)-y)^2 | x \in S] + \Pr_{(x,y) \sim \cD}[x \not\in S] \E_{(x,y) \sim \cD}[(f^*(x)-y)^2 | x \not \in S] \\
&=&  \E_{(x,y) \sim \cD}[(f^*(x)-y)^2]
\end{eqnarray*}
In other words, $f$ is not Bayes optimal. So if we can demonstrate that $f$ is multicalibrated with respect to $\cH$ and $\cD$ we are done. Suppose otherwise. Then there exists some $h \in \cH$ and some $v \in R(f)$ such that 
$$\E_{(x,y) \sim \cD}[h(x)(y-v) | f(x) = v] > 0$$
By Theorem \ref{thm:improve-multical-equivalent}, there exists some $h' \in \cH$ such that:
$$\E_{(x,y) \sim \cD}[(h'(x) - y)^2 | f(x) = v] < \E_{(x,y) \sim \cD}[(f(x) - y)^2 | f(x) = v] $$

We first observe that it must be that $v = c(S)$. If this were not the case, by definition of $f$ we would have that:
$$\E_{(x,y) \sim \cD}[(h'(x) - y)^2 | f(x) = v] < \E_{(x,y) \sim \cD}[(f^*(x) - y)^2 | f(x) = v]$$
which would contradict the Bayes optimality of $f^*$. 
Having established that $v = c(S)$ we can calculate:
\begin{eqnarray*}
&& \E_{(x,y) \sim \cD}[(h'(x) - y)^2 | f(x) = c(S)] \\
&=& \Pr_{(x,y) \sim \cD} [x \in S] \E_{(x,y) \sim \cD}[(h'(x) - y)^2 | x \in S] + \\ && \Pr_{(x,y) \sim \cD} [x \not \in S, f(x) = c(S)] \E_{(x,y) \sim \cD}[(h'(x) - y)^2 | x \not \in S, f(x) = c(S)] \\
&\geq& \Pr_{(x,y) \sim \cD} [x \in S] \E_{(x,y) \sim \cD}[(h'(x) - y)^2 | x \in S] + \\ 
&& \Pr_{(x,y) \sim \cD} [x \not \in S, f(x) = c(S)] \E_{(x,y) \sim \cD}[(f(x) - y)^2 | x \not \in S, f(x) = c(S)] 
\end{eqnarray*}
where in the last inequality we have used the fact that by definition, $f(x) = f^*(x)$ for all $x \not\in S$, and so is pointwise Bayes optimal for all $x \not \in S$. 

Hence the only way we can have $\E_{(x,y) \sim \cD}[(h'(x) - y)^2 | f(x) = c(S)] < \E_{(x,y) \sim \cD}[(f(x) - y)^2 | f(x) = c(S)] $ is if:
$$\E_{(x,y) \sim \cD}[(h'(x) - y)^2 | x \in S] < \E_{(x,y) \sim \cD}[(c(S) - y)^2 | x \in S]   $$
But this contradicts our assumption that $\cH$ violates the weak learning condition on $S$, which completes the proof. 

\end{proof}
\fi 

\ifarxiv
We now turn our attention to deriving a relationship between approximate multicalibration and approximate Bayes optimality. To do so, we'll introduce an even weaker weak learning condition that has one additional parameter $\rho$, lower bounding the mass of sets $S$ that we can condition on while still requiring the weak learning condition to hold. We remark that Algorithm \ref{alg:regression-multicalibrator} can be analyzed as a boosting algorithm under this weaker weak learning assumption as well, with only minor modifications in the analysis.

\begin{definition}[ $(\gamma, \rho)$-weak learning condition]\label{def:wkl-condition} 
Fix a distribution $\cD \in \Delta \cZ$ and let $\cH$ be a class of arbitrary real-valued functions. 
We say that $\cH$ satisfies the \emph{$(\gamma, \rho)$-weak learning condition} for $\cD$ if the following holds. 
For every set $S \subseteq \X$ such that $\Pr_{x \sim \cD_{\cX}}[x \in S] > \rho$, if 
$$\E_{(x,y)\sim D}[(f^{*} - y)^2 \mid x \in S] < \E_{(x,y)\sim D}[(\bar{y}_S - y)^2 \mid x \in S] - \gamma,$$
where $\bar{y}_S = \E_{(x,y)\sim \dist}[y \mid x \in S]$,  
then there exists $h\in \hyp$ such that 
$$\E_{(x,y)\sim D}[(h(x) - y)^2 \mid x \in S] < \E_{(x,y)\sim D}[(\bar{y}_S - y)^2 \mid x \in S] - \gamma.$$
\end{definition}

We may now prove our theorem showing that approximate multicalibration with respect to a class $\cH$ implies approximate Bayes optimality if and only if $\cH$ satisfies the $(\gamma, \rho)$-weak learning condition. We recall Remark~\ref{rem:hbound}, which notes that we must restrict approximate multicalibration to a bounded subset of $\cH$, as we will assume that $\cH$ is closed under affine transformation. 
\begin{theorem}
Fix any distribution $\cD \in \Delta \cZ$, any model $f: \cX \rightarrow [0,1]$, and any class of real valued functions $\cH$ that is closed under affine transformation. 
Let:
$$\cH_1 = \{h \in \cH : \max_{x\in \cX} h(x)^2 \leq 1\}$$
be the set of functions in $\cH$ upper-bounded by $1$ on $\cX$. 
Let $m = |R(f)|$, $\gamma >0$, and $\alpha \leq \frac{\gamma^3}{16m}$. 
Then if $\cH$ satisfies the $(\gamma, \gamma/m)$-weak learning condition and $f$ is $\alpha$-approximately multicalibrated with respect to $\hyp_1$ on $\dist$, then $f$ has squared error $$\E_{(x,y)\sim \dist}[(f(x) - y)^2] \leq \E_{(x,y)\sim \dist}[(f^* - y)^2] + 3\gamma .$$ 
Conversely, if $\hyp$ does not satisfy the $(\gamma, \gamma/m)$-weak learning condition, there exists a model $f: \cX \rightarrow [0,1]$ that is $\alpha$-approximately multicalibrated with respect to $\hyp_1$ on $\dist$, for $\alpha = \gamma$, and is perfectly calibrated on $\cD$, but $f$ has squared error $$\E_{(x,y)\sim \dist}[(f(x) - y)^2] \geq \E_{(x,y)\sim \dist}[(f^* - y)^2] + \gamma^2/m.$$
\end{theorem}

\begin{proof}
We begin by arguing that $\alpha$-approximate multicalibration with respect to $\hyp_1$ on $\dist$ implies approximate Bayes optimality when $\hyp$ satisfies the $(\gamma, \gamma/m)$-weak learning condition. 
Suppose not, and there exists a function $f$ that is $\alpha$-multicalibrated with respect to $\hyp_1$, but 
$$\E_{(x,y)\sim \dist}[(f^* - y)^2] < \E_{(x,y)\sim \dist}[(f(x) - y)^2] - 3\gamma .$$
Then there must exist some $v \in R(f)$ such that $\Pr_{(x,y)\sim \dist}[f(x) =v ] > \gamma/m$ and
$$\E_{(x,y)\sim \dist}[(f^* - y)^2  \mid f(x) = v ] < \E_{(x,y)\sim \dist}[(f(x) - y)^2 \mid f(x) = v] - 2\gamma .$$
We  observe that since $\cH$ is closed under affine transformation, the constant function $h(x) = 1$ is in $\cH$, and so $\alpha$-approximate multicalibration with respect to $\cH_1$ implies $\alpha$-approximate calibration as well.  Thus by definition,
$$\Pr[f(x) = v]\cdot\left(\E_{(x,y)\sim \cD}[v - y \mid f(x) = v] \right)^2 \leq \alpha.$$
Letting $\bar{y}_v = \E[y \mid f(x) = v]$, our lower-bound that $\Pr[f(x) = v] > \gamma/m$ gives us that $(v - \bar{y}_v )^2 < \alpha m/\gamma \leq \left(\frac{\gamma}{4}\right)^2$.
We now use this upper-bound on calibration error in conjuction with our lower-bound on distance from Bayes optimality to show that the squared error of the constant predictor $\bar{y}_v$ must also be far from Bayes optimal. 
\begin{align*} 
\E_{(x,y)\sim \dist}[(f^{*}(x) - y)^2 \mid f(x) = v] 
& < \E_{(x,y)\sim \dist}[(f(x) - y)^2 \mid f(x) = v] - 2\gamma \\
& = \E_{(x,y)\sim \dist}[(v - \bar{y}_v + \bar{y}_v - y)^2 \mid f(x) = v] - 2\gamma \\
&= \E_{(x,y)\sim \dist}[(\bar{y}_v - y)^2 \mid f(x) = v] + (v - \bar{y}_v)^2  - 2\gamma \\
&< \E_{(x,y)\sim \dist}[(\bar{y}_v - y)^2 \mid f(x) = v] - \gamma.
\end{align*}
The $(\gamma, \gamma/m)$-weak learning condition then guarantees that there exists some $h \in \hyp$ such that 
$$\E_{(x,y)\sim D}[(h - y)^2 \mid f(x) =v] < \E_{(x,y)\sim D}[(\bar{y}_v - y)^2 \mid f(x) = v] - \gamma.$$
By Lemma~\ref{lem:improvementToMultical}, the fact that $h$ improves on the squared loss of $\bar{y}_v$ by an additive factor $\gamma$, on the set of $x$ such that $f(x) =v$, implies that $\E[h(x)(y - \bar{y}_v)\mid f(x) =v] > \gamma/2$. 
Because $f$ is $\alpha$-approximately calibrated on $\cD$, we can use the existence of such an $h$ to witness a failure of multicalibration:
\begin{align*}
\E[h(y - v) &\mid f(x) = v] \\
&= \E[h(x)(y - \bar{y}_v + \bar{y}_v - v) \mid f(x) = v] \\
&= \E[h(x)(y - \bar{y}_v) \mid f(x) = v] + \E[h(x)(\bar{y}_v - v) \mid f(x) = v] \\
&> \gamma/2 - \left|\bar{y}_v - v \right| \\
&> \gamma/4.
\end{align*}
Then 
$$ \Pr[f(x) = v]\cdot \left(\E_{(x,y)\sim \dist}[h(x)(y - v)\mid f(x) = v]\right)^2 > \frac{\gamma^3}{16m},$$
contradicting our assumption that $f$ is $\alpha$-approximately multicalibrated with respect to $\cH_1$ for $\alpha < \frac{\gamma^3}{16m}$.
Therefore approximate multicalibration with respect to $\hyp_1$ must imply that $f$ is approximately Bayes optimal. 

It remains to show the other direction, that $\alpha$-approximate multicalibration with respect to a class $\hyp_1$ implies approximate Bayes optimality only if $\hyp$ satisfies the $(\gamma, \gamma/m)$-weak learning condition. If this claim were not true for the stated parameters, then there must exist a class $\hyp$ such that every predictor $f$ that:
\begin{itemize}
\item is $\alpha$-approximately multicalibrated with respect to $\hyp_1$
\item is perfectly calibrated on $\cD$
\item has range with cardinality $|R(f)| = m$
\end{itemize}
also has squared error within $\gamma^2/m$ of Bayes optimal, but $\hyp$ does not satisfy the weak learning condition. 
We will show that no such class exists by defining, for any class $\hyp$ not satisfying the weak learning condition, a predictor $f$ that is $\alpha$-approximately multicalibrated with respect to that class, but has squared error that is not within $\gamma^2/m$ of Bayes optimal. 

Recall that if a class $\hyp$ does not satisfy the $(\gamma, \gamma/m)$-weak learning condition, then there must be some set $S_{\hyp}$ such that $\Pr[x \in S_{\hyp}] > \gamma/m$, there does not exist an $h \in \hyp$ such that 
$$\E_{(x,y)\sim \cD}[(h - y)^2 \mid x \in S_{\hyp}] < \E_{(x,y)\sim \cD}[(\bar{y}_{S_{\hyp}} - y)^2 \mid x \in S_{\hyp}] - \gamma,$$
but for the Bayes optimal predictor, it holds that its squared loss satisfies
$$\E_{(x,y)\sim D}[(f^{*} - y)^2 \mid x \in S_{\hyp}] < \E_{(x,y)\sim D}[(\bar{y}_{S_{\hyp}} - y)^2 \mid x \in S_{\hyp}] - \gamma,$$
where $\bar{y}_{S_{\hyp}} = \E[y \mid x \in S_{\hyp}]$.
For some hypothesis class $\hyp$ not satisfying the weak learning condition, and associated set $S_{\hyp}$, let $f_{\hyp}$ be defined as follows:
$$f_{\hyp}(x) = 
\begin{cases} 
f^{*}(x), & x \not\in S_{\hyp} \\
\bar{y}_{S_{\hyp}}, & x \in S_{\hyp}.
\end{cases}.$$
Note that, because $f_{\hyp}$ is constant on $S_{\hyp}$, there must be some $v \in R(f)$ such that the level set $S_v = \{x \in \X : f(x) = v\}$ contains $S_{\hyp}$. 
To see that $f_{\hyp}$ is $\alpha$-approximately multicalibrated with respect to $\hyp_1$, we first consider the contribution to multicalibration error from the level sets not containing $S_{\hyp}$. For all $h \in \hyp$ and $v \in R(f)$ such that $v \neq \bar{y}_{S_{\hyp}}$,
\begin{align*}
\E_{(x,y)\sim \dist}[h(x)(y - f_{\hyp}(x)) \mid f_{\hyp}(x) = v] 
& = \E_{(x,y)\sim \dist}[h(x)(y - f^{*}(x)) \mid f_{\hyp}(x) = v] \\
& = \E_{x\sim \dist_x}\E_{y\sim \dist_y(x)}[h(x)y \mid f_{\hyp}(x) = v] - \E_{x\sim \dist_x}[h(x)f^{*}(x) \mid f_{\hyp}(x) = v] \\
& = \E_{x\sim \dist_x}\E_{y\sim \dist_y(x)}[h(x)y \mid f_{\hyp}(x) = v] - \E_{x\sim \dist_x}\E_{y\sim \dist_y(x)}[h(x)y \mid f_{\hyp}(x) = v] \\
& = 0.
\end{align*}

For the level set $S_v$ for which $S_{\hyp} \subseteq S_v$, we know from the argument above that the elements $x \in S_v\setminus S_{\hyp}$ contribute nothing to the multicalibration error, as $f(x) = f^{*}(x)$ on these elements. So,
\begin{align*}
\E_{(x,y)\sim \dist}[h(x)(y - f_{\hyp}(x)) \mid f(x) = v] 
& = \Pr_{x\sim \cD_{\cX}}[x \in S_{\hyp}]\cdot\E_{(x,y)\sim \dist}[h(x)(y - \bar{y}_{S_{\hyp}}) \mid x \in S_{\hyp}] \\
& \quad \quad \quad + \Pr_{x\sim \cD_{\cX}}[x \not\in S_{\hyp}]\cdot \E_{(x,y)\sim \dist}[h(x)(y - f^{*}(x)) \mid x \in S_v \setminus S_{\hyp}] \\
& = \Pr_{x\sim \cD_{\cX}}[x \in S_{\hyp}]\cdot\E_{(x,y)\sim \dist}[h(x)(y - \bar{y}_{S_{\hyp}}) \mid x \in S_{\hyp}]
\end{align*}
Therefore if $f_{\cH}$ is not $\alpha$-approximately multicalibrated with respect to $\cH_1$ on $\cD$, it must be the case that there exists some $h\in \hyp_1$ such that $\E[h(x)(y - \bar{y}_{S_{\hyp}}) \mid x \in S_{\hyp}] > \sqrt{\alpha}$. Then by Theorem~\ref{thm:improve-multical-equivalent}, there must exist a $h'\in \hyp$ such that $$\E_{(x,y)\sim \dist}[(\bar{y}_{S_{\hyp}} - y)^2 - (h'(x) - y)^2 \mid x \in S_{\cH}] > \alpha = \gamma.$$
But $S_{\cH}$ was defined to be a subset of $\X$ for which no such $h'$ exists and for which $\Pr[x \in S_{\cH}] > \gamma/m$. This would contradict our assumption that $\hyp$ does not satisfy the $(\gamma, \gamma/m)$-weak learning condition on $\cD$, and therefore $f_{\hyp}$ is $\alpha$-approximately multicalibrated with respect to $\hyp_1$ on $\dist$. 

It remains to prove that $f_{\hyp}$ is far from Bayes optimal. 
\begin{align*}
\E_{(x,y)\sim \dist}[(f_{\hyp}(x) - y)^2] 
& = \Pr_{x \sim \cD_{\cX}}[x \in S_{\hyp}] \E_{(x,y)\sim \cD}[(\bar{y}_{S_{\hyp}} - y)^2 \mid x \in S_{\hyp}] +  \Pr[x \not\in S_{\hyp}]\E_{(x,y)\sim \cD}[(f^{*}(x) - y)^2 \mid x \not\in S_{\hyp}] \\
& \geq  \Pr_{x \sim \cD_{\cX}}[x \in S_{\hyp}] \left(\E_{(x,y)\dist}[(f^{*} - y)^2 \mid x \in S_{\hyp}] + \gamma \right) +  \Pr[x \not\in S_{\hyp}] \E_{(x,y)\sim \cD}[(f^{*}(x) - y)^2 \mid x \not\in S_{\hyp}] \\
& = \E_{(x,y) \sim \cD}[(f^{*} - y)^2] + \gamma\Pr_{x \sim \cD_{\cX}}[x \in S_{\hyp}] \\
& \geq \E_{(x,y)\sim \cD}[(f^{*} - y)^2] + \gamma^2/m.
\end{align*}
\end{proof}
\else
\fi

\ifarxiv
\section{Weak Learners With Respect to Constrained Classes}
\label{sec:constrained}

\ifarxiv
Thus far we have studied function classes $\cH$ that satisfy a weak learning condition with respect to the Bayes optimal predictor $f^*$. But we can also study function classes $\cH$ that satisfy a weak learning condition defined with respect to another constrained class of real valued functions. 

\begin{definition}[Weak Learning Assumption Relative to $\cC$]
\label{def:weaklearner-C}
Fix a distribution $\cD \in \Delta \cZ$ and two classes of functions $\cH$ and $\cC$.  We say that $\cH$ satisfies the $\gamma$-\emph{weak learner} condition relative to $\cC$ and $\cD$ if for every $S \subseteq \cX$ with $\Pr_{x \sim \cD_\cX}[x \in S] > 0$, if:
$$\min_{c \in \cC}\E_{(x,y)\sim D}[(c(x) - y)^2 \mid x \in S] < \E_{(x,y)\sim D}[(\bar{y}_S - y)^2 \mid x \in S] - \gamma,$$
where $\bar{y}_S = \E_{(x,y)\sim \dist}[y \mid x \in S]$,  
then there exists $h\in \hyp$ such that 
$$\E_{(x,y)\sim D}[(h(x) - y)^2 \mid x \in S] < \E_{(x,y)\sim D}[(\bar{y}_S - y)^2 \mid x \in S] - \gamma.$$
When $\gamma = 0$ we simply say that $\cH$ satisfies the weak learning condition relative to $\cC$ and $\cD$.
\end{definition}

We will show that if a predictor $f$ is multicalibrated with respect to $\cH$, and $\cH$ satisfies the weak learning assumption with respect to $\cC$, then in fact:
\begin{enumerate}
\item $f$ is multicalibrated with respect to $\cC$, and
\item $f$ has squared error at most that of the minimum error predictor in $\cC$. 
\end{enumerate}
In fact, \cite{gopalan2022omnipredictors} show that if $f$ is multicalibrated with respect to $\cC$, then it is an \emph{omnipredictor} for $\cC$, which implies that $f$ has loss no more than the best function $c(x) \in \cC$, where loss can be measured with respect to any Lipschitz convex loss function (not just squared error). Thus our results imply that to obtain an omnipredictor for $\cC$, it is sufficient to be multicalibrated with respect to a class $\cH$ that satisfies our weak learning assumption with respect to $\cC$. 

\begin{theorem}
Fix a distribution $\cD \in \Delta \cZ$ and two classes of functions $\cH$ and $\cC$ that are closed under affine transformations. Then if $f:\cX\rightarrow [0,1]$ is multicalibrated with respect to $\cD$ and $\cH$, and if $\cH$ satisfies the weak learning condition relative to $\cC$ and $\cD$, then in fact $f$ is multicalibrated with respect to $\cD$ and $\cC$ as well.
\end{theorem}
\begin{proof}
We assume for simplicity that $f$ has a countable range (which is without loss of generality e.g. whenever $\cX$ is countable). 
Suppose for contradiction that $f$ is not multicalibrated with respect to $\cC$ and $\cD$. In this case there must be some $c \in \cC$ such that:
$$\sum_{v \in R(f)}\Pr[f(x) = v]\left(\E_{(x,y) \sim \cD}[c(x)(y-v)|f(x)=v]\right)^2 > 0$$
Since $\cC$ is closed under affine transformations (and so both $c$ and $-c$ are in $\cC$), there must be some $c' \in \cC$ and some $v \in R(f)$ with $\Pr[f(x) = v] > 0$ such that:
$$\E_{(x,y) \sim \cD}[c'(x)(y-v)|f(x)=v] > 0$$
Therefore, by the first part of Theorem \ref{thm:improve-multical-equivalent}, there must be some $c'' \in \cC$ such that:
$$\E_{(x,y) \sim \cD}[(c''(x)-y)^2 | f(x) = v] < \E_{(x,y) \sim \cD}[(v-y)^2 | f(x) = v] $$
Since $\cH$ is closed under affine transformations, the function $h(x) = 1$ is in $\cH$ and so multicalibration with respect to $\cH$ implies calibration. Thus $v = \bar y_{S_v}$ for $S_v = \{x : f(x) = v\}$. Therefore, the fact that $\cH$ satisfies the weak learning condition relative to $\cC$ and $\cD$ implies that there must be some $h \in \cH$ such that:
$$\E_{(x,y) \sim \cD}[(h(x)-y)^2 | f(x) = v] < \E_{(x,y) \sim \cD}[(v-y)^2 | f(x) = v]$$
Finally, the second part of Theorem \ref{thm:improve-multical-equivalent} implies that:
$$\E_{(x,y) \sim \cD}[h(x)(y-v)|f(x)=v] > 0$$
which is a violation of our assumption that $f$ is multicalibrated with respect to $\cH$ and $\cD$, a contradiction. 
\end{proof}

\begin{theorem}
Fix a distribution $\cD \in \Delta \cZ$ and two classes of functions $\cH$ and $\cC$. Then if $f:\cX\rightarrow [0,1]$ is calibrated and multicalibrated with respect to $\cD$ and $\cH$, and if $\cH$ satisfies the weak learning condition relative to $\cC$ and $\cD$, then:
$$\E_{(x,y) \sim \cD}[(f(x)-y)^2] \leq \min_{c \in \cC}\E_{(x,y) \sim \cD}[(c(x)-y)^2] $$
\end{theorem}
\begin{proof}
We assume for simplicity that $f$ has a countable range (which is without loss of generality e.g. whenever $\cX$ is countable). Suppose for contradiction that there is some $c \in \cC$ such that:
$$\E_{(x,y) \sim \cD}[(c(x)-y)^2] 
 < \E_{(x,y) \sim \cD}[(f(x)-y)^2] $$
 Then there must be some $v \in R(f)$ with $\Pr[f(x) = v] > 0$ and:
$$\E_{(x,y) \sim \cD}[(c(x)-y)^2 | f(x) = v] 
 < \E_{(x,y) \sim \cD}[(v-y)^2 | f(x) = v] $$
 Since $f$ is calibrated, $v = \bar y_{S_v}$ for $S_v = \{x : f(x) = v\}$. Therefore, the fact that $\cH$ satisfies the weak learning condition relative to $\cC$ and $\cD$ implies that there must be some $h \in \cH$ such that:
$$\E_{(x,y) \sim \cD}[(h(x)-y)^2 | f(x) = v] < \E_{(x,y) \sim \cD}[(v-y)^2 | f(x) = v]$$
Finally, the second part of Theorem \ref{thm:improve-multical-equivalent} implies that:
$$\E_{(x,y) \sim \cD}[h(x)(y-v)|f(x)=v] > 0$$
which is a violation of our assumption that $f$ is multicalibrated with respect to $\cH$ and $\cD$, a contradiction.
\end{proof}
We now turn to approximate versions of these statements.
\else
 In this section we prove approximate versions of the exact statements from Section~\ref{sec:omniprediction}, showing that approximate multicalibration with respect to a class $\cH$ implies approximate multicalibration with respect to $\cC$ when $\cH$ satisfies a weak learning condition relative to $\cC$.
 \fi 
 To do so, we need a refined version of one direction of Theorem \ref{thm:improve-multical-equivalent} that shows us that if $f$ witnesses a failure of multicalibration with respect to some $h \in \cH$, then there is another function $h' \in \cH$ that can be used to improve on $f$'s squared error, \emph{while controlling the norm} of $h'$.
\begin{lemma}
Suppose $\cH$ is closed under affine transformation. Fix a model $f:\cX\rightarrow [0,1]$, a levelset $v \in R(f)$, and a bound $B > 0$. Then if there exists an $h \in \cH$ such that $\max_{x \in \cX} h(x)^2 \leq B$ and
$$\E[h(x)(y-v)| f(x) = v] \geq \alpha,$$
for $\alpha \geq 0$, then there exists an $h' \in \cH$ such that $\max_{x \in \cX} h'(x)^2 \leq (1 + \frac{\sqrt{B}}{\alpha})^2$ and:
$$\E[(f(x)-y)^2 - (h'(x) - y)^2 | f(x) = v] \geq \frac{\alpha^2}{B}.$$
\end{lemma}
\begin{proof}
Let $h'(x) = v + \eta h(x)$ where $\eta = \frac{\alpha}{\E[h(x)^2 \mid f(x) = v]}$, as in Theorem~\ref{thm:improve-multical-equivalent}. Because $h(x)^2$ is uniformly bounded by $B$ on $\cX$, it follows that $\E[h(x)^2] \leq B$, and we have already shown in the proof of Theorem~\ref{thm:improve-multical-equivalent} that this implies 
$$\E[(f(x)-y)^2 - (h'(x) - y)^2 | f(x) = v] \geq \frac{\alpha^2}{B}.$$
It only remains to bound $\max_{x \in \cX} h'(x)^2$. We begin by lower-bounding $\E[h(x)^2 \mid f(x) = v]$ in terms of $\alpha$.
\begin{align*}
\E[h(x)^2 \mid f(x) = v]
&\geq \E[h(x) \mid f(x) = v]^2 \\
&\geq \E[h(x)(y - v) \mid f(x) = v]^2\\
&\geq \alpha^2. 
\end{align*}
It follows that $\eta \leq 1/\alpha$, and so
\begin{align*}
\max_{x \in \cX} h'(x)^2
& = \max_{x \in \cX} (v + \eta h(x))^2 \\
&\leq (1 + \eta \sqrt{B})^2 \\
& \leq \left(1 + \frac{\sqrt{B}}{\alpha}\right)^2.
\end{align*}
\end{proof}

We will also need a parameterized version of our weak learning condition. Recalling Remark~\ref{rem:hbound}, for approximate multicalibration to be meaningful with respect to a class that is closed under affine transformation, we must specify a bounded subset of that class with respect to which a predictor is approximately multicalibrated. Then to show that approximate multicalibration with respect to one potentially unbounded class implies approximate multicalibration with respect to another, we will need to specify the subsets of each class with respect to which a predictor is claimed to be approximately multicalibrated. This motivates a parameterization of our previous weak learning condition relative to a class $\cC$. We will need to assume that whenever there is a $B$-bounded function in $\cC$ that improves over the best constant predictor on a restriction of $\cD$, there also exists a $B$-bounded function in $\cH$ that improves on the restriction as well.
\begin{definition}[$B$-Bounded Weak Learning Assumption Relative to $\cC$]
\label{def:weaklearner-C-B}
Fix a distribution $\cD \in \Delta \cZ$ and two classes of functions $\cH$ and $\cC$. Fix a bound $B >0 $ and let $\cH_B$ and $\cC_B$ denote the sets
$$\cH_B = \{ h \in \cH : \max_{x \in \cX} h(x)^2 \leq B \}$$ and $$\cC_B = \{ c \in \cC : \max_{x \in \cX} c(x)^2 \leq B \}$$ respectively. We say that $\cH$ satisfies the $B$-bounded $\gamma$-\emph{bounded weak learning} condition relative to $\cC$ and $\cD$ if for every $S \subseteq \cX$ with $\Pr_{x \sim \cD_\cX}[x \in S] > 0$, if:
$$\min_{c \in \cC_B}\E_{(x,y)\sim D}[(c(x) - y)^2 \mid x \in S] < \E_{(x,y)\sim D}[(\bar{y}_S - y)^2 \mid x \in S] - \gamma,$$
where $\bar{y}_S = \E[y \mid x \in S]$,  
then there exists $h\in \cH_B$ such that 
$$\E_{(x,y)\sim D}[(h(x) - y)^2 \mid x \in S] < \E_{(x,y)\sim D}[(\bar{y}_S - y)^2 \mid x \in S] - \gamma.$$
\end{definition}

\begin{theorem}\label{thm:approx-wkl-mc-transference}
Fix a distribution $\cD \in \Delta \cZ$ and two classes of functions $\cH$ and $\cC$ that are closed under affine transformations. Fix $\alpha_{\cC}, B > 0$. Let $B' = (1 + \sqrt{\frac{2B}{\alpha_{\cC}}})^2$ and $\gamma = \frac{\alpha_{\cC}}{4B}$.  
Fix a function $f:\cX \rightarrow [0,1]$ that maps into a countable subset of its range, and let $m = |R(f)|$, $\alpha_{\cH} < \frac{\alpha_{\cC}^3}{2^9m{B'}^2}$, and $\alpha < \frac{\alpha_{\cC}\gamma^2}{32m{B'}^2}$. Then if 
\begin{itemize}
\item $\cH$ satisfies the $B'$-bounded $\gamma$-weak learning condition relative to $\cC$ and $\cD$
\item $f$ is $\alpha_{\cH}$-approximately multicalibrated with respect to $\cD$ and $\cH_{B'}$
\item $f$ is $\alpha$-approximately calibrated on $\cD$,
\end{itemize}
then $f$ is $\alpha_{\cC}$-approximately multicalibrated with respect to $\cD$ and $\cC_{B}$.
\end{theorem}
\begin{proof}
Suppose not and there exists some $c \in \cC_{B}$ such that 
$$\sum_{v \in R(f)}\Pr_{x \sim \cD_x}[f(x) = v]\cdot\left(\E_{(x,y)\sim \cD}[c(x)(y - v) \mid f(x) = v] \right)^2 > \alpha_{\cC}.$$
 Then there must exist some $v \in R(f)$ such that
$\Pr[f(x) = v] > \frac{\alpha_{\cC}}{2m}$ and 
$$\E_{(x,y)\sim \cD}[c(x)(y - v) \mid f(x) = v]^2 > \alpha_{\cC}/2.$$
Because $\cC$ is closed under affine transformations, $\cC_{B}$ is closed under negation, so there must also exist some $c' \in \cC_{B}$ such that
$$ \E_{(x,y)\sim \cD}[c'(x)(y - v) \mid f(x) = v] > \sqrt{\alpha_{\cC}/2}.$$
Then Lemma~\ref{lem:multicalToImprovement} shows that there is a $c'' \in \cC_{(1 + \sqrt{\frac{2B}{\alpha_{\cC}}})^2} = \cC_{B'}$ such that 
$$\E_{(x,y)\sim \cD}[(y - f(x))^2 - (y - c''(x))^2 \mid f(x) = v] \geq \frac{\alpha_{\cC}}{2B} = 2\gamma.$$
Because $f$ is $\alpha$-calibrated on $\cD$, by definition we have
$$\Pr_{x\sim \cD_x}[f(x) = v]\cdot\left( \E_{(x,y)\sim \cD}[v - y \mid f(x) = v]\right)^2 < \alpha.$$
Letting $\bar{y}_v = \E[y \mid f(x) = v]$, our lower-bound that $\Pr[f(x) = v] > \frac{\alpha_{\cC}}{2m}$ gives us that $(v - \bar{y}_v)^2 < \frac{2\alpha m}{\alpha_{\cC}} \leq \frac{\gamma^2}{16{B'}^2} < \gamma$.
So, because $v$ is close to $\bar{y}_v$, we can show the squared error of $f$ must be close to the squared error of $\bar{y}_v$ on this level set.
\begin{align*}
\E_{(x,y)\sim \cD}[(y - f(x))^2 \mid f(x) = v] 
&= \E_{(x,y)\sim \cD}[(y - \bar{y}_v + \bar{y}_v - f(x))^2 \mid f(x) = v] \\
&= \E_{(x,y)\sim \cD}[(y - \bar{y}_v)^2 + 2(y - \bar{y}_v)(\bar{y}_v - v)\mid f(x) = v] + (\bar{y}_v - v)^2  \\
&= \E_{(x,y)\sim \cD}[(y - \bar{y}_v)^2 \mid f(x) = v] + (\bar{y}_v - v)^2 \\
&< \E_{(x,y)\sim \cD}[(y - \bar{y}_v)^2 \mid f(x) = v] + \gamma.
\end{align*}
Then, because the squared error of $c''$ on this level set is much less than the squared error of $f$, we find that $c''$ must also have squared error less than that of $\bar{y}_v$:
\begin{align*}
\E_{(x,y)\sim \cD}[(y - \bar{y}_v)^2 - (y - c''(x))^2 \mid f(x) = v]
& > \E_{(x,y)\sim \cD}[(y - f(x))^2 - \gamma - (y - c''(x))^2 \mid f(x) = v] \\
& \geq 2\gamma - \gamma \\
&= \gamma
\end{align*}
We assumed $\cH$ satisfies the $B'$-bounded $\gamma$-weak learning condition relative to $\cC$, so this gives us a function $h \in \cH_{B'}$ such that
$$\E_{(x,y)\sim \cD}[(y - \bar{y}_v)^2 - (y - h(x))^2 \mid f(x) = v] > \gamma .$$
Then Lemma~\ref{lem:multicalToImprovement} shows that 
$$\E[h(x)(y - \bar{y}_v)\mid f(x) = v] > \gamma/2.$$
So $h$ witnesses a failure of multicalibration of $f$, since it follows that
\begin{align*}
\E[h(x)(y-v) \mid f(x) = v]
&= \E[h(x)(y-\bar{y}_v) \mid f(x) = v] + \E[h(x)(\bar{y}_v - v) \mid f(x) = v]\\
&> \gamma/2 - B'\left|\bar{y}_v - v \right| \\
&\geq \gamma/2 - \frac{B'\gamma}{4B'} \\
&= \gamma/4
\end{align*}
and so
$$\Pr_{x\sim \cD_x}[f(x) = v]\left(\E_{(x,y)\sim \cD}[h(x)(y-v) \mid f(x) = v] \right)^2 > \frac{\alpha_{\cC}\gamma^2}{32m} > \alpha_{\cH},$$
contradicting $\alpha_{\cH}$-approximate multicalibration of $f$ on $\cH_{B'}$ and $\cD$.
\end{proof}

In \cite{gopalan2022omnipredictors}, Gopalan, Kalai, Reingold, Sharan, and Wieder show that any predictor that is approximately multicalibrated for a class $\cH$ and distribution $\cD$ can be efficiently post-processed to approximately minimize any convex, Lipschitz loss function relative to the class $\cH$. The theorem we have just proved can now be used to extend their result to approximate loss minimization over any other class $\cC$, so long as $\cH$ satisfies the $B$-bounded $\gamma$-weak learning assumption relative to $\cC$. Intuitively, this follows from the fact that if $f$ is approximately multicalibrated with respect to $\cH$ on $\cD$, it is also approximately multicalibrated with respect to $\cC$. However, the notion of approximate multicalibration adopted in \cite{gopalan2022omnipredictors} differs from the one in this work. So, to formalize our intuition above, we will first state the covariance-based definition of approximate multicalibration appearing in \cite{gopalan2022omnipredictors} and prove a lemma relating it to our own. We note that, going forward, we will restrict ourselves to distributions $\cD$ over $\cX \times \{0,1\}$, as in this case the two definitions of approximate multicalibration are straightforwardly connected. 

\begin{definition}[Approximate Covariance Multicalibration \cite{gopalan2022omnipredictors}]
Fix a distribution $\cD$ over $\cX \times \{0,1\}$ and a function $f: \cX \rightarrow [0,1]$ that maps onto a countable subset of its range, denoted $R(f)$. Let $\cH$ be an arbitrary collection of real valued functions $h: \X \rightarrow \mathbb{R}$. Then $f$ is $\alpha$-\emph{approximately covariance multicalibrated} with respect to $\cH$ on $\cD$ if
$$\sum_{v \in R(f)}\Pr_{x\sim \cD_{\cX}}[f(x) = v]\cdot \left| \E[(h(x) - \bar{h}_v)(y - \bar{y}_v) \mid f(x) = v]\right| \leq \alpha,$$
where $\bar{h}_v = \E[h(x) \mid f(x) = v]$ and $\bar{y}_v = \E[y \mid f(x) = v]$.
\end{definition}

\begin{lemma}\label{lem:Too-Many-MCs-Not-Enough-Mics}
Fix a distribution $\cD$ over $\cX \times \{0,1\}$ and a class of functions on $\cX$, $\cH$. Let $\cH_B$ denote the subset
$$\cH_B = \{h \in \cH : \max_{x \in \cX} h(x)^2 \leq B\}.$$
Fix a function $f: \cX \rightarrow [0,1]$ that maps onto a countable subset of its range, denoted $R(f)$. Then if $f$ is $\alpha$-approximately multicalibrated with respect to $\cH_B$ on $\cD$, then $f$ is $(\sqrt{\alpha}(1 + \sqrt{B}))$-approximately covariance multicalibrated. That is, for all $h\in \cH_B$, $f$ satisfies
$$\sum_{v\in R(f)}\Pr[f(x) = v]\cdot \left|\E[(h(x) - \bar{h}_v)(y - \bar{y}_v)\mid f(x) = v] \right| \leq \sqrt{\alpha}(1+\sqrt{B}).$$
\end{lemma}
\begin{proof}
\begin{align*}
\sum_{v\in R(f)}\Pr[f(x) = v]\cdot &\left|\E[(h(x) - \bar{h}_v)(y - \bar{y}_v)\mid f(x) = v] \right| \\
&= \sum_{v\in R(f)}\Pr[f(x) = v]\cdot \left|\E[h(x)y \mid f(x) = v] - \bar{y}_v\bar{h}_v\right| \\
&= \sum_{v\in R(f)}\Pr[f(x) = v]\cdot \left|\E[h(x)y \mid f(x) = v] -v\bar{h}_v + v\bar{h}_v - \bar{y}_v\bar{h}_v\right| \\
&= \sum_{v\in R(f)}\Pr[f(x) = v]\cdot \left|\E[h(x)(y - v) \mid f(x) = v]  + \bar{h}_v (v- \bar{y}_v)\right| \\
&\leq \sum_{v\in R(f)}\Pr[f(x) = v]\cdot \left(\left|\E[h(x)(y - v) \mid f(x) = v]\right|  + \left|\bar{h}_v (v- \bar{y}_v)\right|\right) \\
&\leq \sqrt{\alpha} + \sqrt{B}\sum_{v\in R(f)}\Pr[f(x) = v]\cdot \left|v- \bar{y}_v\right|\\
&\leq \sqrt{\alpha}(1 + \sqrt{B}).
\end{align*}
where the second inequality follows from the fact that 
$\E[x] \leq \sqrt{\E[x^2]}$ and the bound $\max_{x \in \cX}h(x)^2 \leq B$. 
\end{proof}

We now recall a theorem of \cite{gopalan2022omnipredictors}, showing that approximate covariance multicalibration with respect to a class $\cH$ implies approximate loss minimization relative to $\cH$, for convex, Lipschitz losses. 

\begin{theorem}
Fix a distribution $\cD$ over $\cX \times \{0,1\}$ and a class of real-valued functions on $\cX$, $\cH$. Fix a function $f: \cX \rightarrow [0,1]$ that maps onto a countable subset of its range, denoted $R(f)$. Let $\cL$ be a class of functions on $\{0,1\}\times \mathbb{R}$ that are convex and $L$-Lipschitz in their second argument. If $f$ is $\alpha$-approximately covariance multicalibrated with respect to $\cH_B$ on $\cD$, then for every $\ell \in \cL$ there exists an efficient post-processing function $k_{\ell}$ such that 
$$\E_{(x,y)\sim \cD}[\ell(y, k_{\ell}(f(x)))] \leq \min_{h \in \cH_B}\E_{(x,y)\sim \cD}[\ell(y, h(x))] + 2\alpha L.$$
\end{theorem}

\begin{corollary}
Fix a distribution $\cD$ over $\cX \times \{0,1\}$ and two classes of real-valued functions on $\cX$ that are closed under affine transformation, $\cH$ and $\cC$. Fix a function $f: \cX \rightarrow [0,1]$ that maps onto a countable subset of its range, denoted $R(f)$. Let $\cL$ be a class of functions on $\{0,1\}\times \mathbb{R}$ that are convex and $L$-Lipschitz in their second argument. 
Fix $\alpha_{\cC}, B > 0$. Let $B' = (1 + \sqrt{\frac{2B}{\alpha_{\cC}}})^2$ and $\gamma = \frac{\alpha_{\cC}}{4B}$.  
 Let $\alpha_{\cH} < \frac{\alpha_{\cC}^3}{2^9m{B'}^2}$, and $\alpha < \frac{\alpha_{\cC}\gamma^2}{32m{B'}^2}$. Then if 
\begin{itemize}
\item $\cH$ satisfies the $B'$-bounded $\gamma$-weak learning condition relative to $\cC$ and $\cD$
\item $f$ is $\alpha_{\cH}$-approximately multicalibrated with respect to $\cD$ and $\cH_{B'}$
\item $f$ is $\alpha$-approximately calibrated on $\cD$,
\end{itemize}
then for every $\ell \in \cL$ there exists an efficient post-processing function $k_{\ell}$ such that 
$$\E_{(x,y)\sim \cD}[\ell(y, k_{\ell}(f(x)))] \leq \min_{c \in \cC_B}\E_{(x,y)\sim \cD}[\ell(y, c(x))] + 2L\sqrt{\alpha_{\cC}}(1 + \sqrt{B}).$$
\end{corollary}
\begin{proof}
We have from Theorem~\ref{thm:approx-wkl-mc-transference} that given the assumed conditions, $f$ will be $\alpha_{\cC}$-approximately multicalibrated with respect to $\cC_B$ on $\cD$. It follows from Lemma~\ref{lem:Too-Many-MCs-Not-Enough-Mics} that $f$ is $\sqrt{\alpha_{\cC}}(1+ \sqrt{B})$-approximately covariance multicalibrated with respect to $\cC_B$ on $\cD$. The result of \cite{gopalan2022omnipredictors} then gives us that for all $\ell \in \cL$, there exists an efficient post-processing function $k_{\ell}$ such that 
$$\E_{(x,y)\sim \cD}[\ell(y, k_{\ell}(f(x)))] \leq \min_{c \in \cC_B}\E_{(x,y)\sim \cD}[\ell(y, c(x))] + 2L\sqrt{\alpha_{\cC}}(1 + \sqrt{B}).$$
\end{proof}
\fi 

\section{Empirical Evaluation}
\label{sec:experimental}
\ifarxiv
In this section, we study Algorithm \ref{alg:regression-multicalibrator} empirically via an efficient, open-source Python implementation of our algorithm on both synthetic and real regression problems. Our code is available here: \url{https://github.com/Declancharrison/Level-Set-Boosting}. 
An important feature of Algorithm \ref{alg:regression-multicalibrator} which distinguishes it from traditional boosting algorithms is the ability to parallelize not only during inference, but also during training. Let $f_{t}$ be the model maintained by Algorithm \ref{alg:regression-multicalibrator} at round $t$ with $m$ level sets. Given a data set $X$, $f_{t}$ creates a partition of $X$ defined by $X^{t+1}_{i}= \{x | f_{t}(x) = v_{i}\}$. Since the $X_{i}$ are disjoint, each call $h^{t+1}_{i} = A_{\mathcal{H}}(X^{t+1}_{i})$ can be made on a separate worker followed by a combine and round operation to obtain $\tilde{f}_{t+1}$ and $f_{t+1}$ respectively, as shown in Figure~\ref{fig:diagram}. A parallel inference pass at round $t$ works nearly identically, but uses the historical weak learners $h^{t+1}_{i}$ obtained from training and applies them to each set $X^{t+1}_{i}$. 
\begin{figure}[t]
\centering
\includegraphics[trim = {.6cm .1cm .1cm .1cm}, clip, width=.8\linewidth]{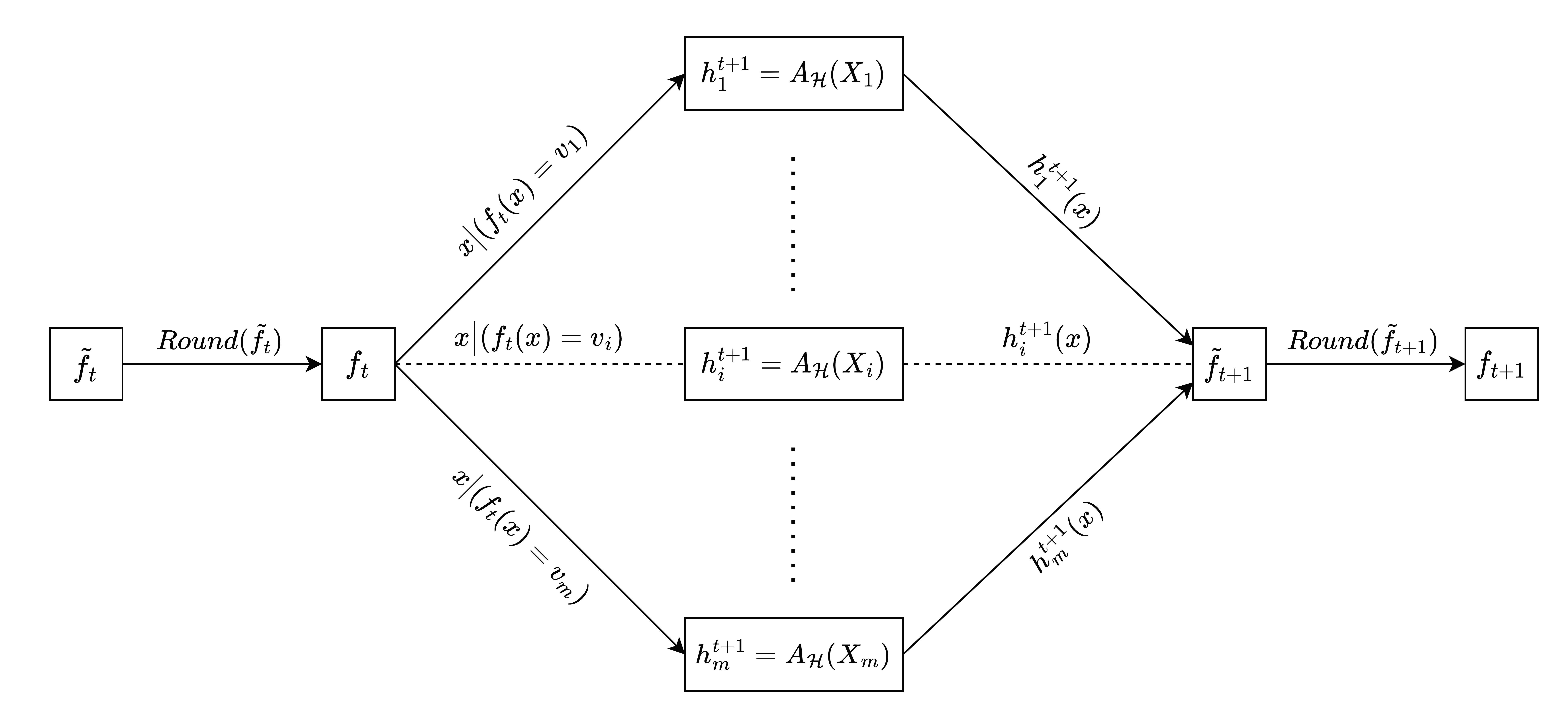}
\caption{The update process at round $t$ with $m$ level sets during training.}
\label{fig:diagram}
\end{figure}

\subsection{Prediction on Synthetic Data}\label{ssec:synthetic-task}
From Theorem \ref{thm:weaklearnercharacterize}, we know that multicalibration with respect to a hypothesis class $\mathcal{H}$ satisfying our weak learning condition implies Bayes optimality. To visualize the fast convergence of our algorithm to Bayes optimality, we create two synthetic datasets; each dataset contains one million samples with two features. We label these points using two functions, \ref{eq:C_0} and \ref{eq:C_1}, defined below and pictured in Figure \ref{fig:bayes-optimal}). We attempt to learn the underlying function with Algorithm \ref{alg:regression-multicalibrator}. 

\[
C_{0}(x) = 
\begin{cases}
(x+1)^{2} + (y - 1)^{2}, & \text{if } x \leq 0, y \geq 0 \\
(x-1)^{2} + (y - 1)^{2}, & \text{if } x > 0, y \geq 0 \\
(x+1)^{2} + (y + 1)^{2}, & \text{if } x \leq 0, y < 0 \\
(x-1)^{2} + (y + 1)^{2}, & \text{if } x > 0, y < 0 \\
\end{cases}\label{eq:C_0} \tag{$C_{0}$}
\]
\[
C_{1}(x) = 
\begin{cases}
x + 20xy^{2}\cos(-8x)\sin(8y)\left(\frac{(1.5x + 4)(x+1)^2}{y+3} + (y-1)^{2}\right), & \text{if } x \leq 0, y \geq 0 \\
x + 20xy^{2}\cos(8x)\sin(8y)\left(\frac{(1.5x + 4)(x-1)^2}{y+3} + (y-1)^{2}\right), & \text{if } x > 0, y \geq 0 \\
x + 20xy^{2}\cos(-8x)\sin(8y)\left(\frac{(1.5x + 4)(x+1)^2}{y+3} + (y+1)^{2}\right), & \text{if } x \leq 0, y < 0 \\
x + 20xy^{2}\cos(8x)\sin(8y)\left(\frac{(1.5x + 4)(x-1)^2}{y+3} + (y+1)^{2}\right), & \text{if } x > 0, y < 0 \\
\end{cases}\label{eq:C_1} \tag{$C_{1}$}
\]
\\

In Figure \ref{fig:simple-learn}, we show an example of Algorithm \ref{alg:regression-multicalibrator} learning \ref{eq:C_0} using a discretization of five-hundred level sets and a weak learner hypothesis class of depth one decision trees. Each image in figure \ref{fig:simple-learn} corresponds to the map produced by Algorithm \ref{alg:regression-multicalibrator} at the round listed in the top of the image. As the round count increases, the number of non-empty level sets increases until each level set is filled, at which point the updates become more granular. The termination round titled `final round' occurs at $T = 199$ and paints an approximate map of \ref{eq:C_0}. The image titled `out of sample' is the map produced on a set of one million points randomly drawn outside of the training sample, and shows that Algorithm \ref{alg:regression-multicalibrator} is in fact an approximation of the Bayes Optimal \ref{eq:C_0}.

\begin{center}
\begin{figure}[H]
\begin{subfigure}[c]{.475\linewidth}
    \centering
    \includegraphics[trim={2.2cm 2.2cm 2.2cm 2.2cm},clip, width = \linewidth]{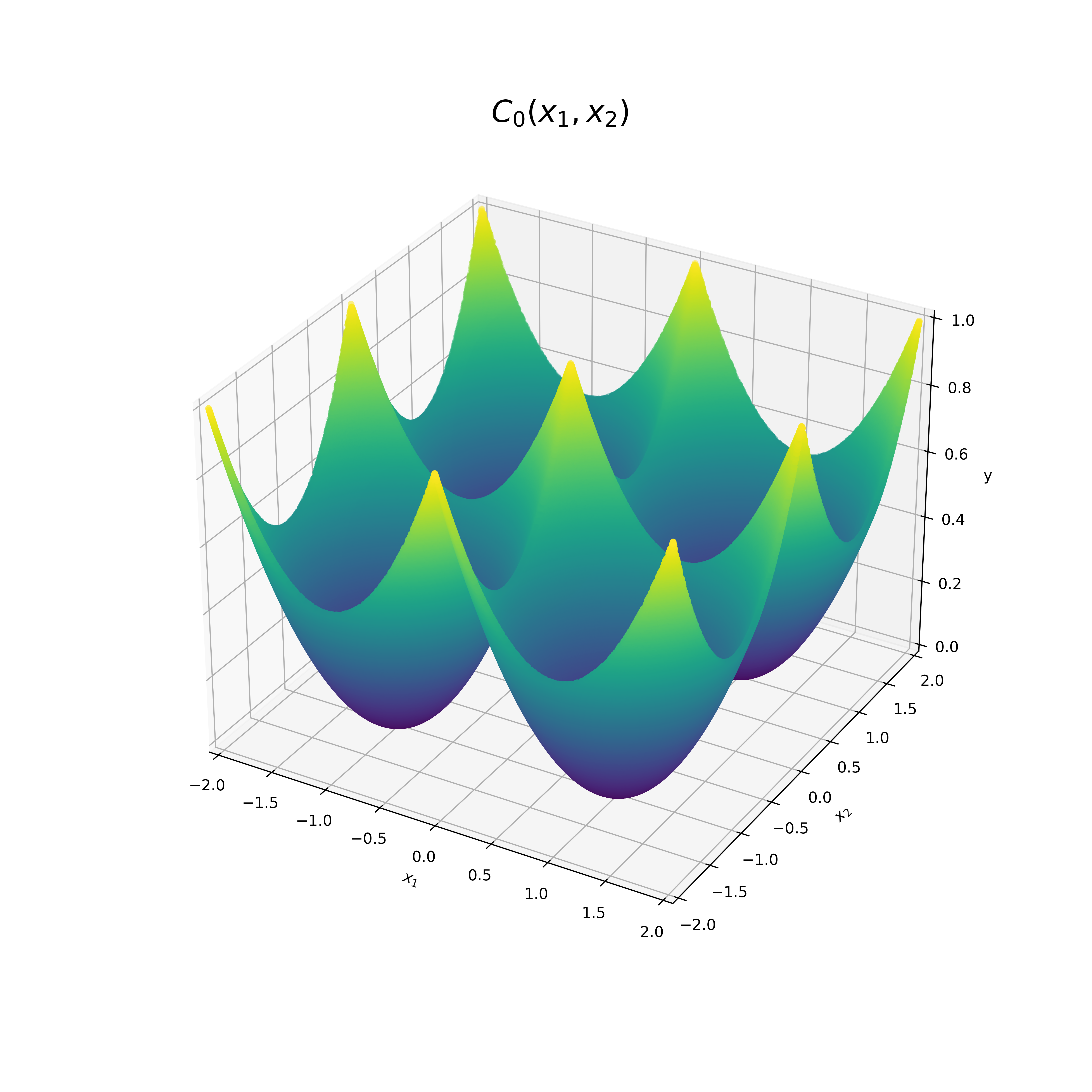}
\end{subfigure}
\begin{subfigure}[c]{.475\linewidth}
    \centering
    \includegraphics[trim={2.2cm 2.2cm 2.2cm 2.2cm},clip,width = \linewidth]{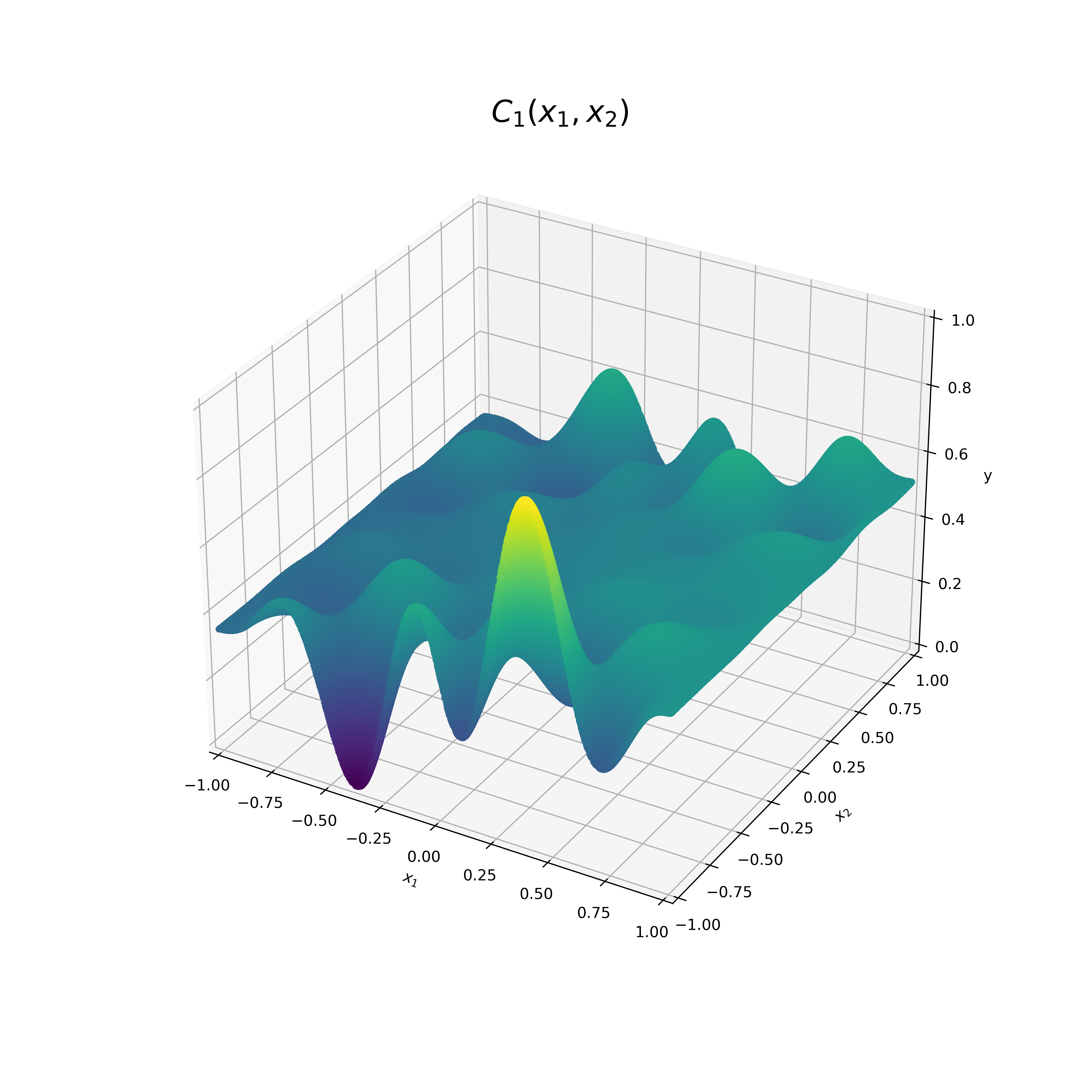}
\end{subfigure}  
\caption{\ref{eq:C_0} maps $x_{1},x_{2}\in[-2,2]$ to four cylindrical cones symmetric about the origin. \ref{eq:C_1} maps $x_{1},x_{2}\in[-1,1]$ to a hilly terrain from a more complex function.}
\label{fig:bayes-optimal}
\end{figure}
\end{center}
\begin{figure}
\begin{center}
\includegraphics[scale = .5, width = .97\linewidth]{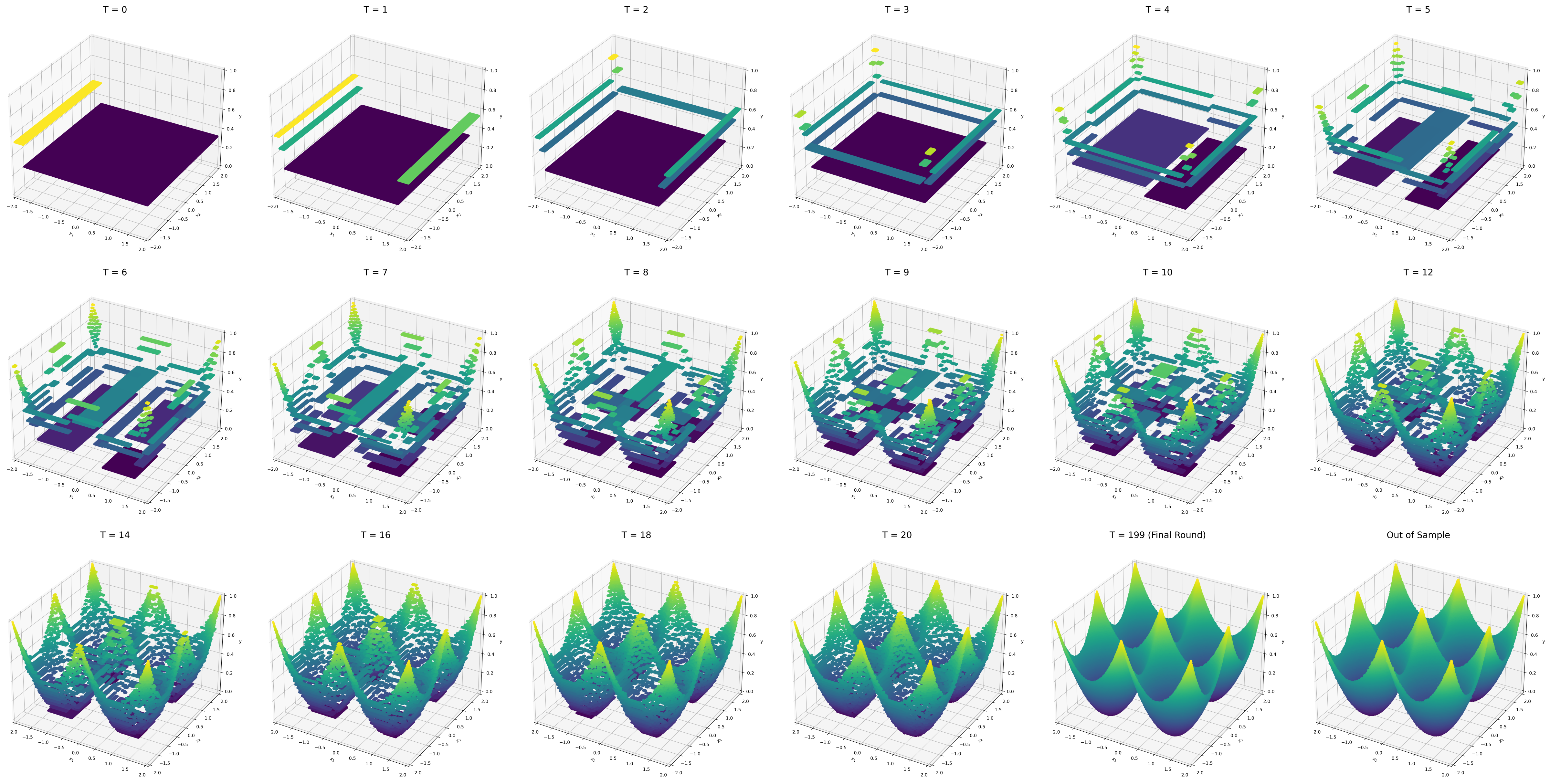}
\caption{Evolution of Algorithm \ref{alg:regression-multicalibrator} learning \ref{eq:C_0}.}
\label{fig:simple-learn}
\end{center}
\end{figure}

Figure \ref{fig:complex-learn} plots the same kind of progression as Figure \ref{fig:simple-learn}, but with a more complicated underlying function \ref{eq:C_1} using a variety of weak learner classes. We are able to learn this more complex surface out of sample with all base classes except for linear regression, which results in a noisy out-of-sample plot. 

\begin{figure}
    \centering
    \includegraphics[width = \linewidth]{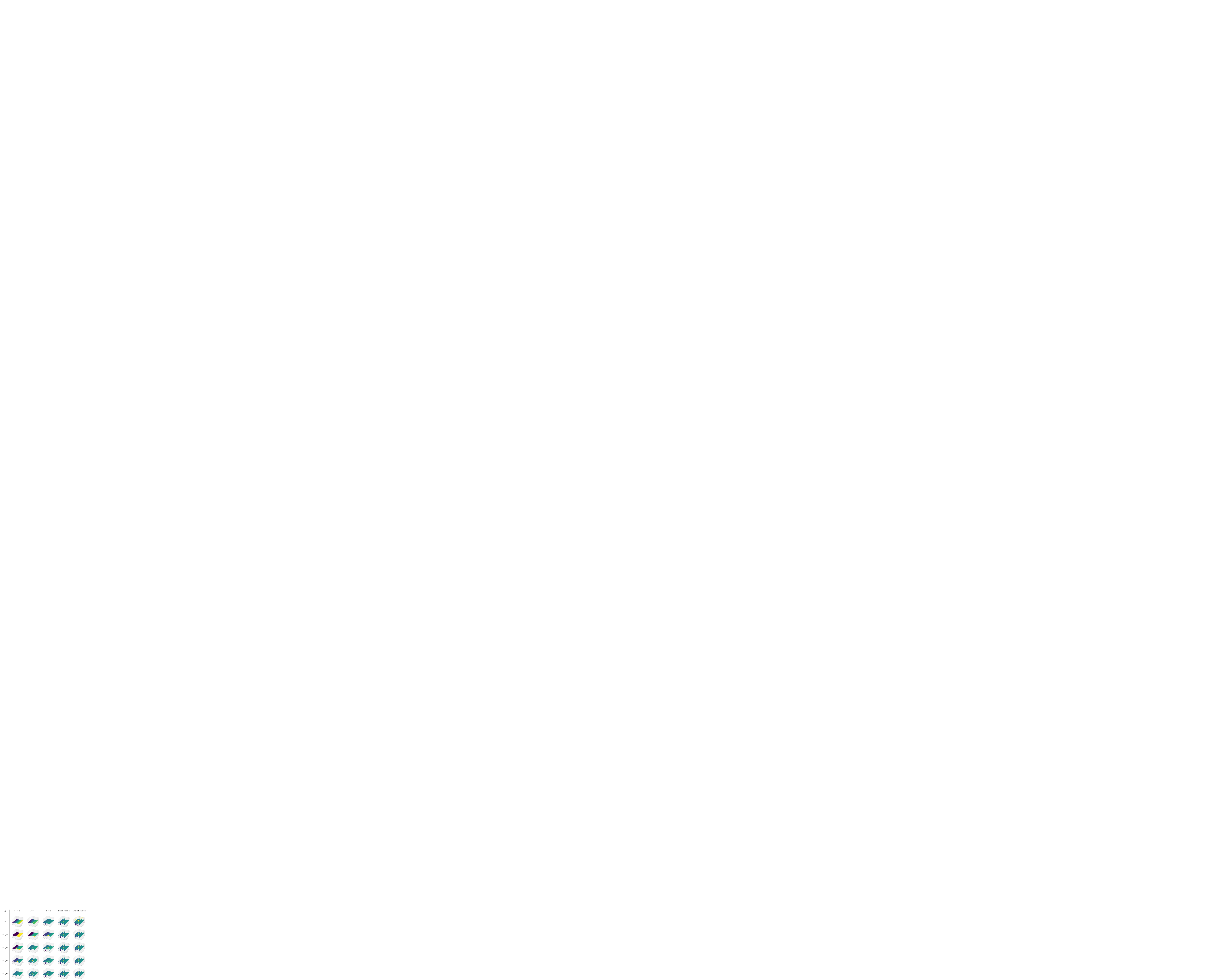}
    \caption{Stages of Algorithm \ref{alg:regression-multicalibrator} learning \ref{eq:C_1} with linear regression (LR) and varying depth $d$ decision trees (DT($d$)). In the out of sample plot for linear regression, points are not mapped to their proper position, implying \ref{eq:C_1} cannot be learned by boosting linear functions. All other hypothesis classes eventually converge to \ref{eq:C_1}.}
    \label{fig:complex-learn}
\end{figure}

\subsection{Prediction on Census Data}\label{ssec:census-task}
We evaluate the empirical performance of Algorithm \ref{alg:regression-multicalibrator} on US Census data compiled using the Python folktables package \cite{ding2021retiring}. In this dataset, the feature space consists of demographic information about individuals (see Table~\ref{tab:feature-space}), and the labels correspond to the individual's annual income. 
\begin{table}[H]
    \centering
    \begin{tabular}{c|c|c|c}
    feature & description & feature & description\\
    \hline
    AGEP & age & POBP & place of birth  \\ 
    COW & class of worker & RELP & relationship\\
    SCHL & education level & WKHP & work hours per week \\ 
    MAR & marital status & SEX & binary sex \\
    OCCP & occupation & RAC1P & race \\ 
    \end{tabular}
    \caption{Features included in income prediction task.}
    \label{tab:feature-space}
\end{table}
We cap income at \$100,000 and then rescale all labels into $[0,1]$. On an 80/20\% train-test split with ~500,000 total samples, we compare the performance of Algorithm \ref{alg:regression-multicalibrator} with Gradient Boosting with two performance metrics: mean squared error (MSE), and mean squared calibration error (MSCE). For less expressive weak learner classes (such as DT(1), see Figure \ref{fig:dt2_folk}), Algorithm \ref{alg:regression-multicalibrator} has superior MSE out of sample compared to Gradient Boosting through one hundred rounds while maintaining significantly lower MSCE, and converges quicker. However, as the weak learning class becomes more expressive (e.g. increasing decision tree depths), Algorithm \ref{alg:regression-multicalibrator} is more prone to overfitting than gradient boosting (see Figure \ref{fig:gb_vs_ls}).

\begin{figure}[H]
\centering
\includegraphics[trim = {15 11 15 11}, clip, width = .77\linewidth]{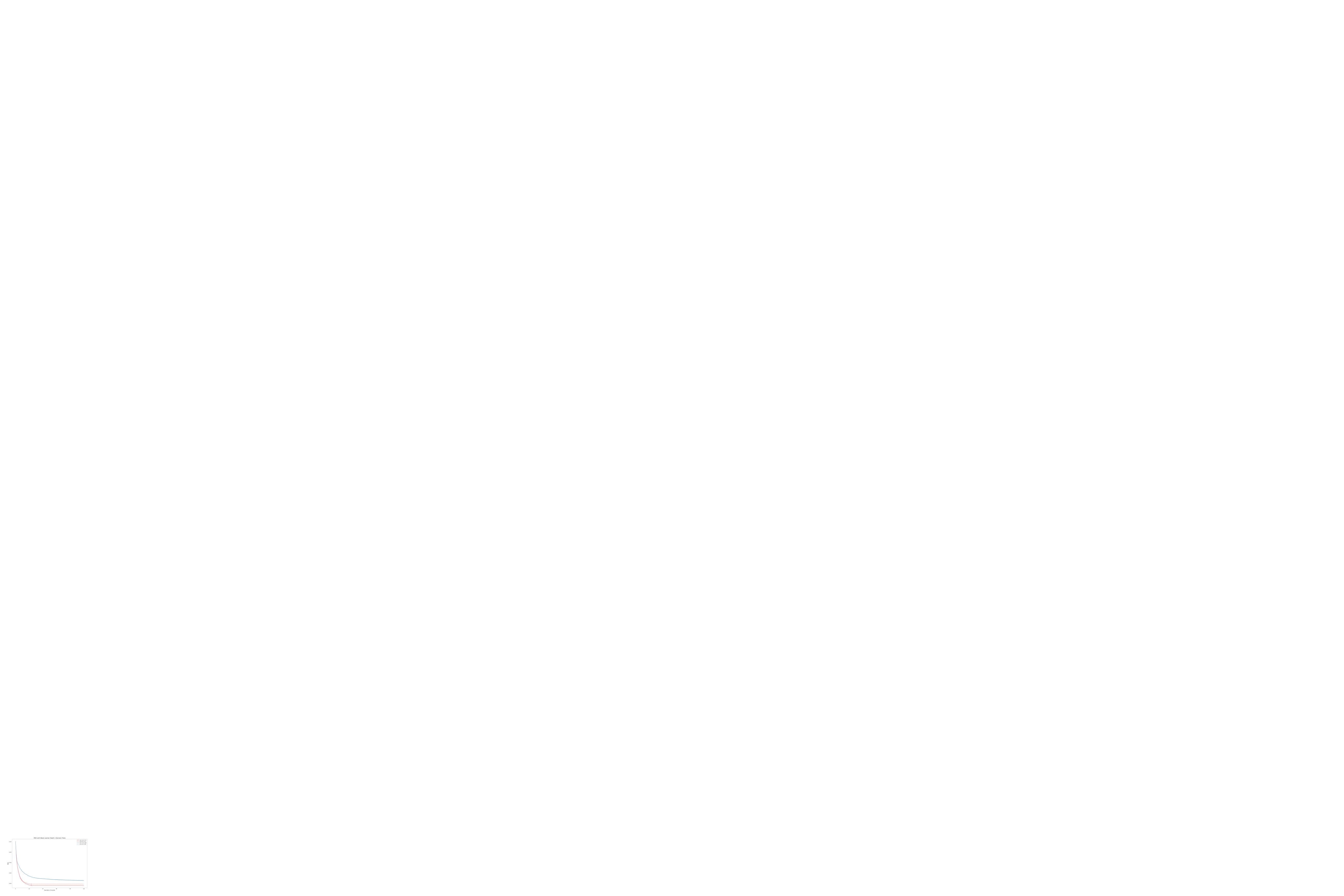} \\
\includegraphics[trim = {15 11 15 11}, clip, width = .77\linewidth]{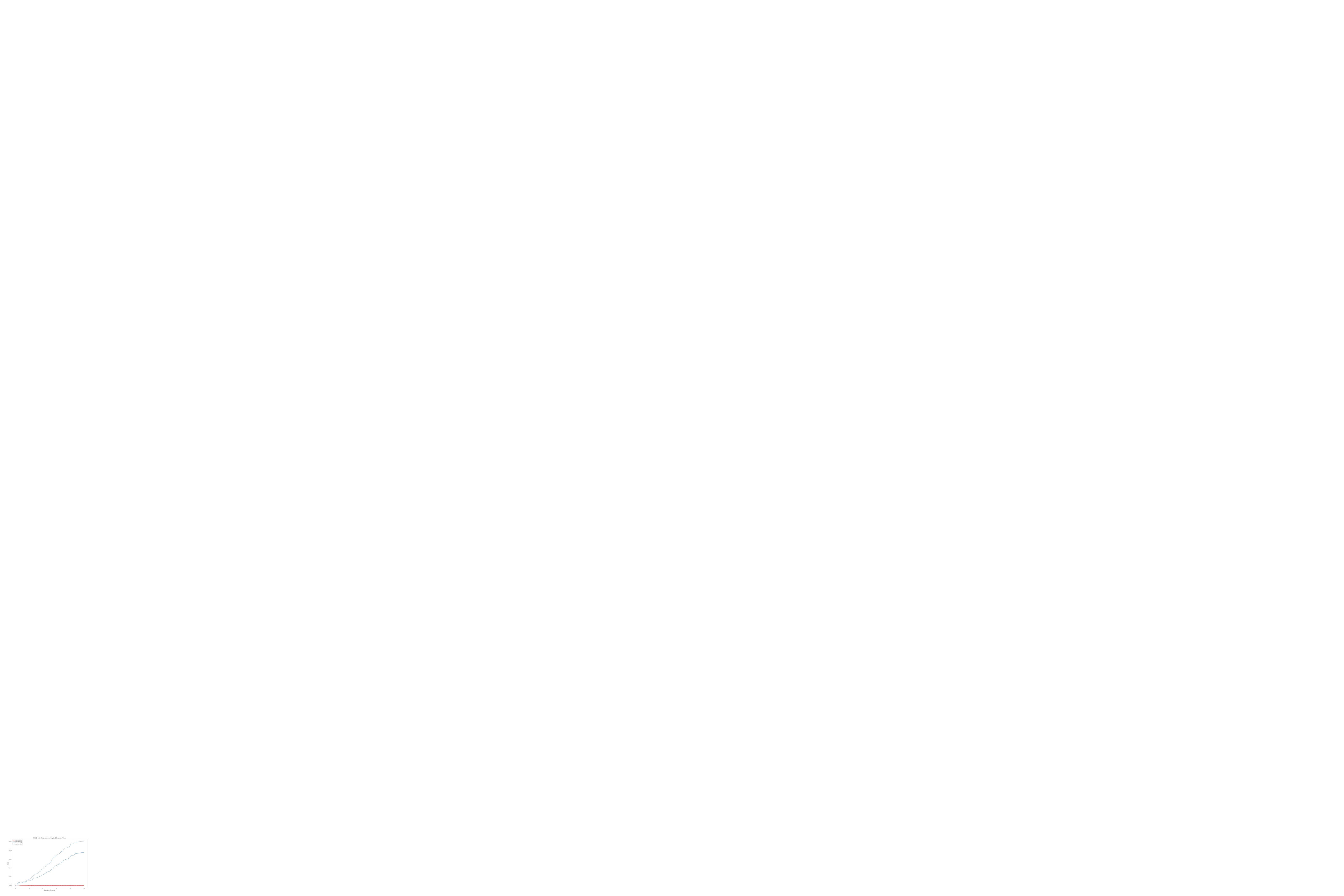}
\caption{Comparison of Algorithm \ref{alg:regression-multicalibrator} (LS) and Gradient Boosting (GB), both using depth 1 regression trees. * indicates termination round of Algorithm \ref{alg:regression-multicalibrator}.}
\label{fig:dt2_folk}
\end{figure}

In Table \ref{tab:ls-gb-times}, we compare the time taken to train $n$ weak learners with Algorithm \ref{alg:regression-multicalibrator} and with scikit-learn's version of Gradient Boosting on our census data. 
Recall that our algorithm trains multiple weak learners per round of boosting, and so comparing the two algorithms for a fixed number of calls to the weak learner is distinct from comparing them for a fixed number of rounds. 
Because models output by Algorithm~\ref{alg:regression-multicalibrator} may be more complex than those produced by Gradient Boosting run for the same number of rounds, we use number of weak learners trained as a proxy for model complexity, and compare the two algorithms holding this measure fixed. 
We see the trend for Gradient Boosting is linear with respect to number of weak learners, whereas Algorithm \ref{alg:regression-multicalibrator} does not follow the same linear pattern upfront. This is due to not being able to fully leverage parallelization of training weak learners in early stages of boosting. At each round, Algorithm~\ref{alg:regression-multicalibrator} calls the weak learner on every large enough level set of the current model, and it is these independent calls that can be easily parallelized. However, in the early rounds of boosting the model may be relatively simple, and so many level sets may be sparsely populated. As the model becomes more expressive over subsequent rounds, the weak learner will be invoked on more sets per round, allowing us to fully utilize parallelizability. 
\begin{table}[H]
    \centering
    \begin{tabular}{|c|c|c|c|c|c|c|c|c|c|}
    \hline
    \multirow{2}{*}{\# Weak Learners} & \multicolumn{3}{c|}{DT(1)} & \multicolumn{3}{c|}{DT(2)} & \multicolumn{3}{c|}{DT(3)}\\
      \cline{2-10} 
      & LS & GB & Faster? & LS & GB & Faster? & LS & GB & Faster?\\ 
    \hline
    \multicolumn{10}{|c|}{50 level sets}\\
    \hline
    100  & 9.11  & 11.97  & \checkmark & 5.86  & 23.01  & \checkmark & 6.88  & 32.92  & \checkmark \\ 
    300  & 18.70 & 35.81  & \checkmark & 14.90 & 69.17  & \checkmark & 15.64 & 102.14 & \checkmark \\ 
    500  & 27.00 & 58.19  & \checkmark & 21.74 & 115.65 & \checkmark & 24.77 & 169.90 & \checkmark \\ 
    1000 & 46.73 & 116.49 & \checkmark & 42.92 & 231.74 & \checkmark & 46.38 & 336.89 & \checkmark \\ 
    \hline
    \multicolumn{10}{|c|}{100 level sets}\\
    \hline 
    100  & 7.18  & 11.97  & \checkmark & 5.29  & 23.01  & \checkmark & 5.06  & 32.92  & \checkmark \\ 
    300  & 13.08 & 35.81  & \checkmark & 13.55 & 69.17  & \checkmark & 14.72 & 102.14 & \checkmark \\ 
    500  & 21.20 & 58.19  & \checkmark & 19.57 & 115.65 & \checkmark & 21.79 & 169.90 & \checkmark \\ 
    1000 & 41.99 & 116.49 & \checkmark & 36.26 & 231.74 & \checkmark & 40.92 & 336.89 & \checkmark \\ 
    \hline
     \multicolumn{10}{|c|}{300 level sets}\\
    \hline 
    100  & 5.87  & 11.97  & \checkmark & 9.18  & 23.01  & \checkmark & 6.54  & 32.92  & \checkmark \\ 
    300  & 13.21 & 35.81  & \checkmark & 17.46 & 69.17  & \checkmark & 11.13 & 102.14 & \checkmark \\ 
    500  & 19.05 & 58.19  & \checkmark & 22.20 & 115.65 & \checkmark & 19.64 & 169.90 & \checkmark \\ 
    1000 & 32.80 & 116.49 & \checkmark & 36.61 & 231.74 & \checkmark & 27.12 & 336.89 & \checkmark \\ 
    \hline
    \end{tabular}
    \caption{Time (in seconds) comparison of Algorithm \ref{alg:regression-multicalibrator} (LS) with fifty level sets and Gradient Boosting to train certain numbers of estimators for various weak learner classes.}
    \label{tab:ls-gb-times}
\end{table}

In Figure \ref{fig:gb_vs_ls}, we measure MSE and MSCE for Algorithm \ref{alg:regression-multicalibrator} and Gradient Boosting over rounds of training on our census data. Again, we note that one round of Algorithm \ref{alg:regression-multicalibrator} is not equivalent to one round of Gradient Boosting, but intend to demonstrate error comparisons and rates of convergence. For the linear regression plots, Gradient Boosting does not reduce either error since combinations of linear models are also linear. As the complexity of the underlying model class increases, Gradient Boosting surpasses Algorithm \ref{alg:regression-multicalibrator} in terms of MSE, though it does not minimize calibration error. 

We notice that Algorithm \ref{alg:regression-multicalibrator}, like most machine learning algorithms, is prone to overfitting when allowed. Future performance hueristics we intend to investigate include validating updates, complexity penalties, and weighted mixtures of updates. 

\begin{figure}[H]
    \centering
    \includegraphics[trim = {15 11 15 11}, clip, width = .45\linewidth]{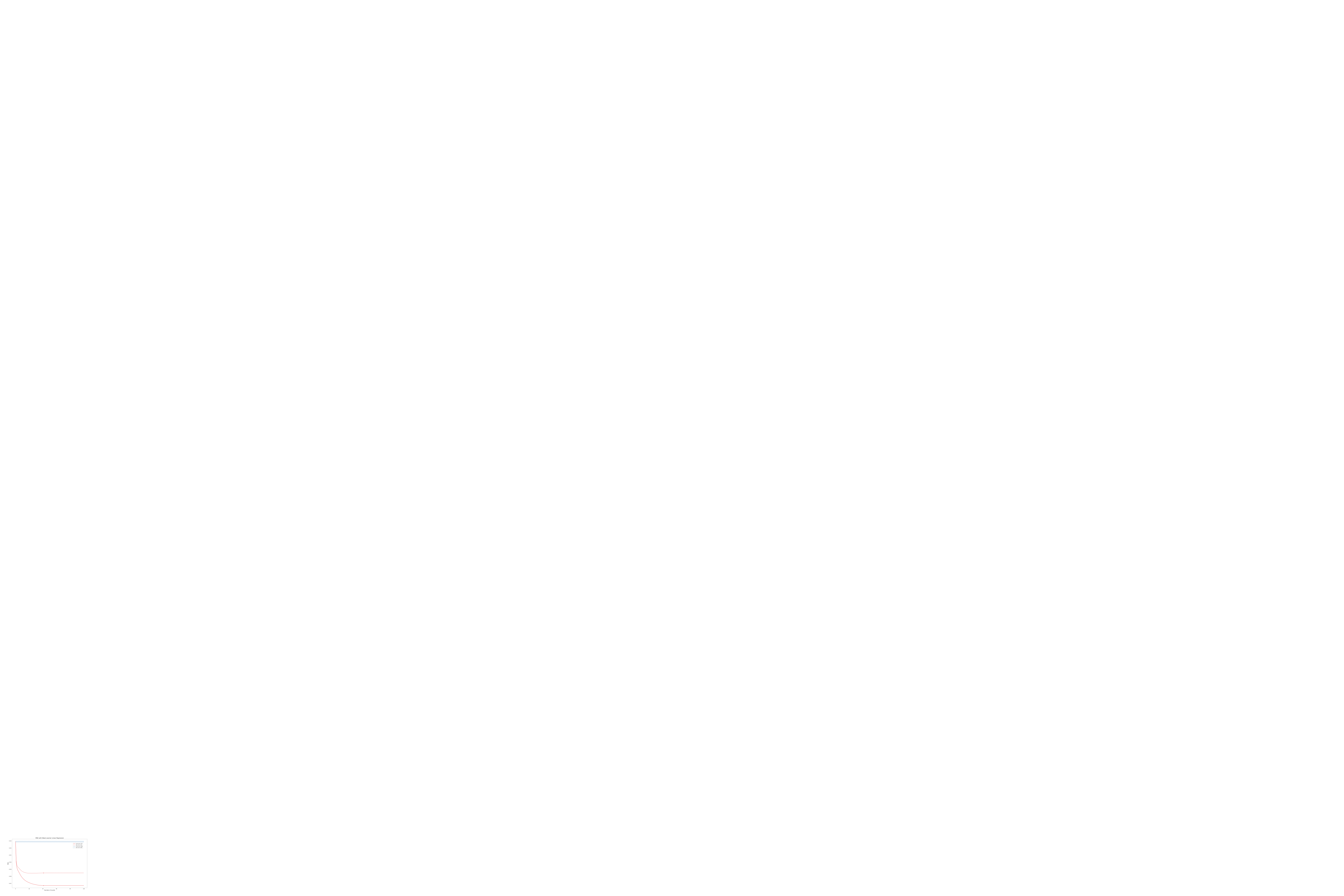}
    \includegraphics[trim = {15 11 15 11}, clip, width = .45\linewidth]{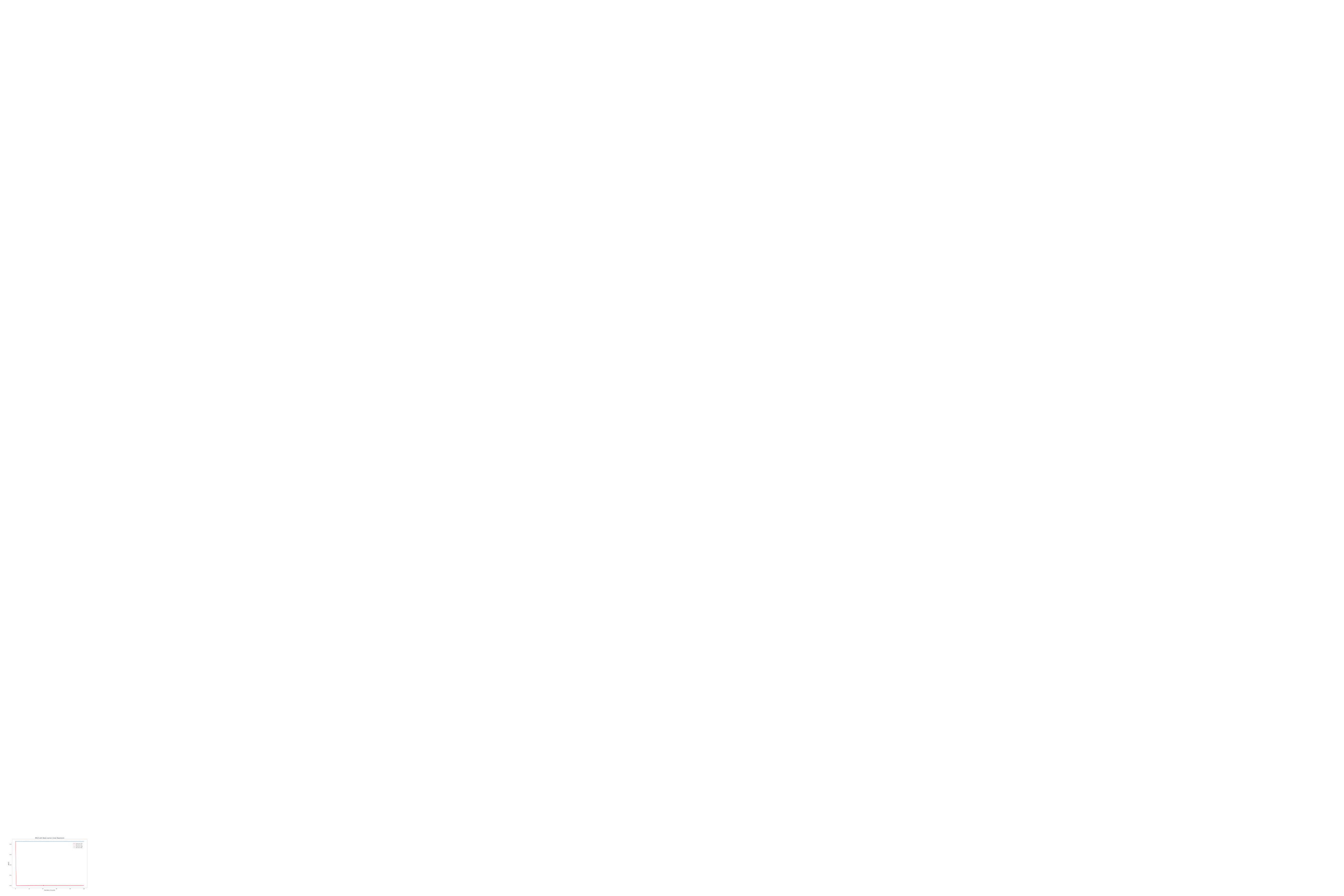}
    \\
    \includegraphics[trim = {15 11 15 11}, clip, width = .45\linewidth]{Experimentation/gb_vs_ls/dt1_mse.pdf}
    \includegraphics[trim = {15 11 15 11}, clip, width = .45\linewidth]{Experimentation/gb_vs_ls/dt1_mse.pdf}
    \\
    \includegraphics[trim = {15 11 15 11}, clip, width = .45\linewidth]{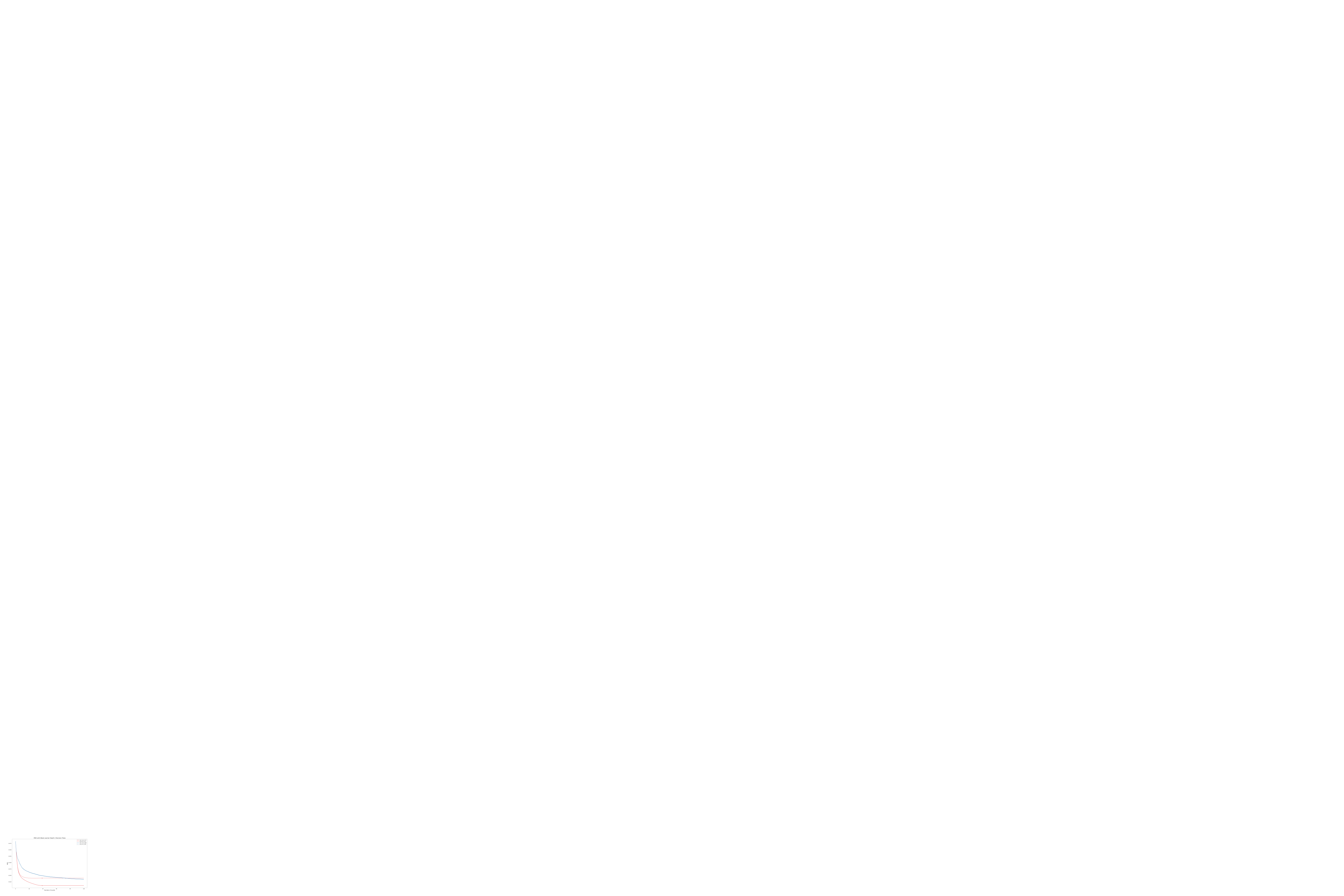}
    \includegraphics[trim = {15 11 15 11}, clip,width = .45\linewidth]{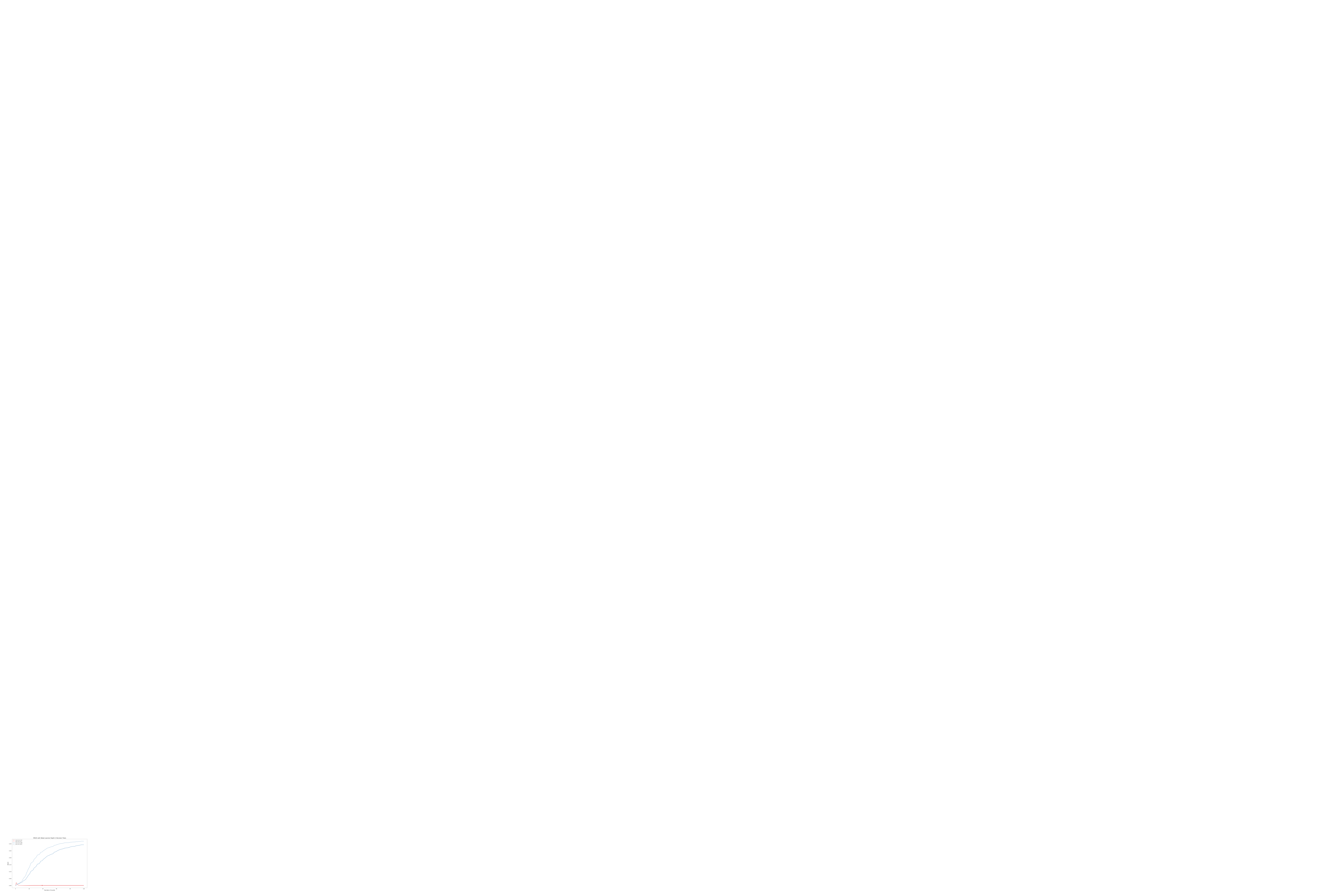}
    \\
    \includegraphics[trim = {15 11 15 11}, clip, width = .45\linewidth]{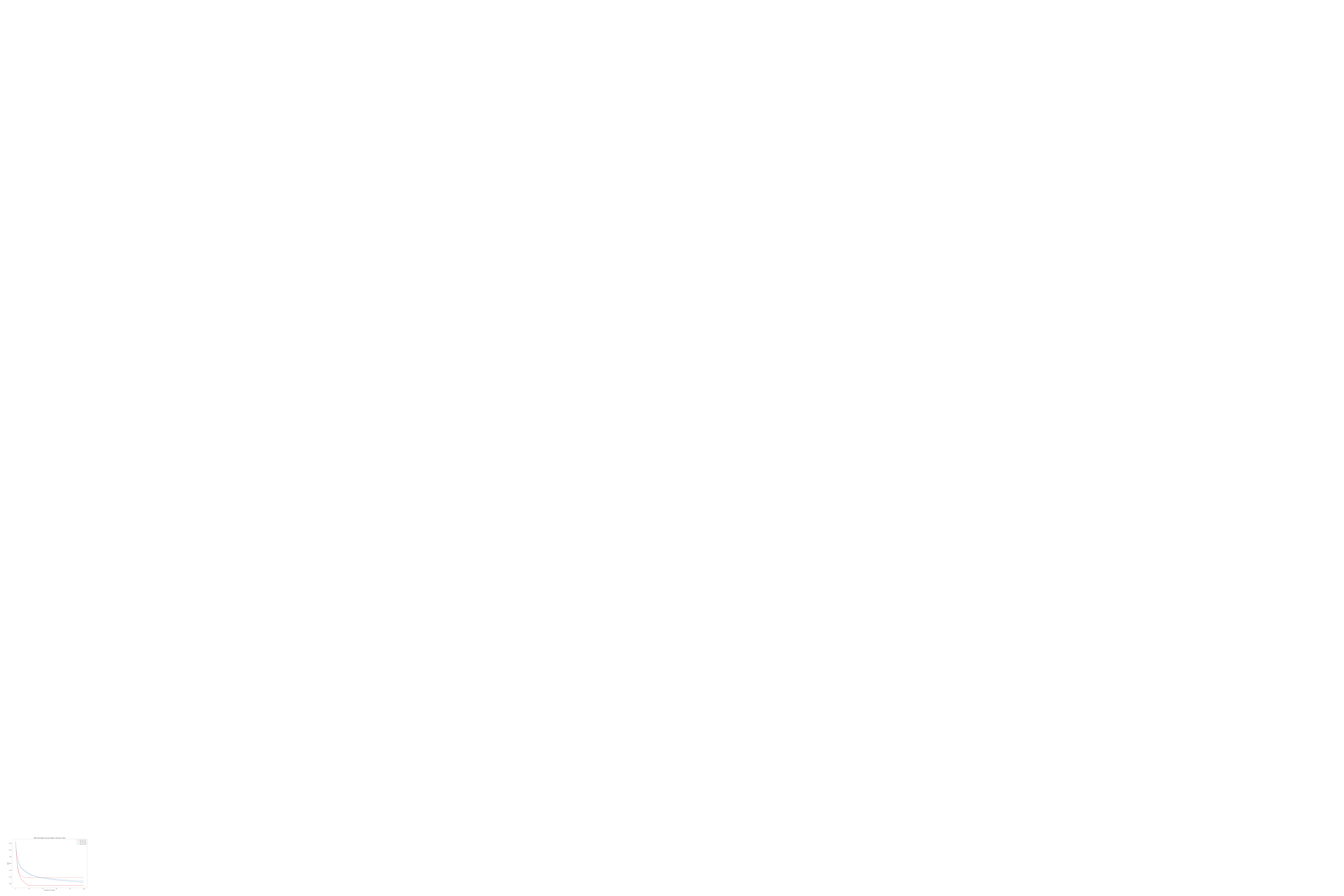}
    \includegraphics[trim = {15 11 15 11}, clip, width = .45\linewidth]{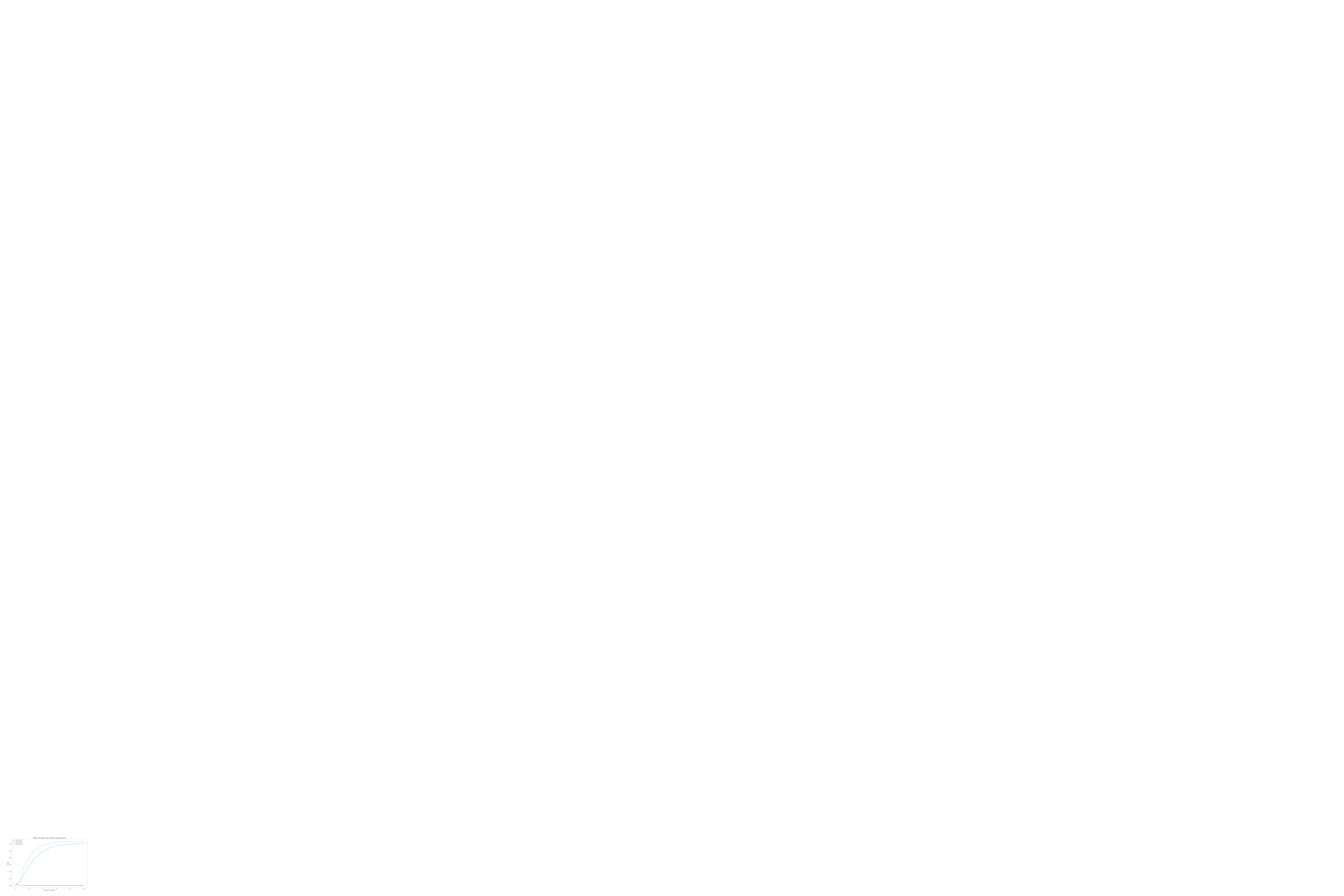}
    \\
    
    \caption{MSE and MSCE comparison of Algorithm \ref{alg:regression-multicalibrator} (LS) and Gradient Boosting (GB) on linear regression and decision trees of varying depths. * indicates termination round of LS and occurs, from top to bottom, at $T = 41, 23, 39, 20$.}
    \label{fig:gb_vs_ls}
\end{figure}


\fi

\bibliographystyle{plainnat}
\bibliography{references}

\begin{thebibliography}{25}
\providecommand{\natexlab}[1]{#1}
\providecommand{\url}[1]{\texttt{#1}}
\expandafter\ifx\csname urlstyle\endcsname\relax
  \providecommand{\doi}[1]{doi: #1}\else
  \providecommand{\doi}{doi: \begingroup \urlstyle{rm}\Url}\fi

\bibitem[Blum and Mansour(2005)]{blum2005external}
Avrim Blum and Yishay Mansour.
\newblock From external to internal regret.
\newblock In \emph{International Conference on Computational Learning Theory},
  pages 621--636. Springer, 2005.

\bibitem[Burhanpurkar et~al.(2021)Burhanpurkar, Deng, Dwork, and
  Zhang]{burhanpurkar2021scaffolding}
Maya Burhanpurkar, Zhun Deng, Cynthia Dwork, and Linjun Zhang.
\newblock Scaffolding sets.
\newblock \emph{arXiv preprint arXiv:2111.03135}, 2021.

\bibitem[Dawid(1982)]{dawid1982well}
A~Philip Dawid.
\newblock The well-calibrated bayesian.
\newblock \emph{Journal of the American Statistical Association}, 77\penalty0
  (379):\penalty0 605--610, 1982.

\bibitem[Ding et~al.(2021)Ding, Hardt, Miller, and Schmidt]{ding2021retiring}
Frances Ding, Moritz Hardt, John Miller, and Ludwig Schmidt.
\newblock Retiring adult: New datasets for fair machine learning.
\newblock \emph{Advances in Neural Information Processing Systems}, 34, 2021.

\bibitem[Duffy and Helmbold(2002)]{duffy2002boosting}
Nigel Duffy and David Helmbold.
\newblock Boosting methods for regression.
\newblock \emph{Machine Learning}, 47\penalty0 (2):\penalty0 153--200, 2002.

\bibitem[Foster and Vohra(1999)]{foster1999regret}
Dean~P Foster and Rakesh Vohra.
\newblock Regret in the on-line decision problem.
\newblock \emph{Games and Economic Behavior}, 29\penalty0 (1-2):\penalty0
  7--35, 1999.

\bibitem[Freund and Schapire(1997)]{freund1997decision}
Yoav Freund and Robert~E Schapire.
\newblock A decision-theoretic generalization of on-line learning and an
  application to boosting.
\newblock \emph{Journal of computer and system sciences}, 55\penalty0
  (1):\penalty0 119--139, 1997.

\bibitem[Friedman(2001)]{friedman2001greedy}
Jerome~H Friedman.
\newblock Greedy function approximation: a gradient boosting machine.
\newblock \emph{Annals of statistics}, pages 1189--1232, 2001.

\bibitem[Gopalan et~al.(2022)Gopalan, Kalai, Reingold, Sharan, and
  Wieder]{gopalan2022omnipredictors}
Parikshit Gopalan, Adam~Tauman Kalai, Omer Reingold, Vatsal Sharan, and Udi
  Wieder.
\newblock Omnipredictors.
\newblock In \emph{ITCS}, 2022.

\bibitem[H{\'e}bert-Johnson et~al.(2018)H{\'e}bert-Johnson, Kim, Reingold, and
  Rothblum]{hebert2018multicalibration}
Ursula H{\'e}bert-Johnson, Michael Kim, Omer Reingold, and Guy Rothblum.
\newblock Multicalibration: Calibration for the (computationally-identifiable)
  masses.
\newblock In \emph{International Conference on Machine Learning}, pages
  1939--1948. PMLR, 2018.

\bibitem[Jung et~al.(2021)Jung, Lee, Pai, Roth, and Vohra]{jung2021moment}
Christopher Jung, Changhwa Lee, Mallesh Pai, Aaron Roth, and Rakesh Vohra.
\newblock Moment multicalibration for uncertainty estimation.
\newblock In \emph{Conference on Learning Theory}, pages 2634--2678. PMLR,
  2021.

\bibitem[Jung et~al.(2022)Jung, Noarov, Ramalingam, and Roth]{jung2022batch}
Christopher Jung, Georgy Noarov, Ramya Ramalingam, and Aaron Roth.
\newblock Batch multivalid conformal prediction.
\newblock \emph{arXiv preprint arXiv:2209.15145}, 2022.

\bibitem[Kalai(2004)]{kalai2004learning}
Adam Kalai.
\newblock Learning monotonic linear functions.
\newblock In \emph{International Conference on Computational Learning Theory},
  pages 487--501. Springer, 2004.

\bibitem[Kalai et~al.(2008)Kalai, Klivans, Mansour, and
  Servedio]{kalai2008agnostically}
Adam~Tauman Kalai, Adam~R Klivans, Yishay Mansour, and Rocco~A Servedio.
\newblock Agnostically learning halfspaces.
\newblock \emph{SIAM Journal on Computing}, 37\penalty0 (6):\penalty0
  1777--1805, 2008.

\bibitem[Kanade and Kalai(2009)]{kanade2009potential}
Varun Kanade and Adam Kalai.
\newblock Potential-based agnostic boosting.
\newblock \emph{Advances in neural information processing systems}, 22, 2009.

\bibitem[Kim et~al.(2019)Kim, Ghorbani, and Zou]{kim2019multiaccuracy}
Michael~P Kim, Amirata Ghorbani, and James Zou.
\newblock Multiaccuracy: Black-box post-processing for fairness in
  classification.
\newblock In \emph{Proceedings of the 2019 AAAI/ACM Conference on AI, Ethics,
  and Society}, pages 247--254, 2019.

\bibitem[Kim et~al.(2022)Kim, Kern, Goldwasser, Kreuter, and
  Reingold]{kim2022universal}
Michael~P Kim, Christoph Kern, Shafi Goldwasser, Frauke Kreuter, and Omer
  Reingold.
\newblock Universal adaptability: Target-independent inference that competes
  with propensity scoring.
\newblock \emph{Proceedings of the National Academy of Sciences}, 119\penalty0
  (4):\penalty0 e2108097119, 2022.

\bibitem[Natarajan(1989)]{natarajan1989learning}
Balas~K Natarajan.
\newblock On learning sets and functions.
\newblock \emph{Machine Learning}, 4\penalty0 (1):\penalty0 67--97, 1989.

\bibitem[Pollard(2012)]{pollard2012convergence}
David Pollard.
\newblock \emph{Convergence of stochastic processes}.
\newblock Springer Science \& Business Media, 2012.

\bibitem[Roth(2022)]{rothuncertain}
Aaron Roth.
\newblock Uncertain: Modern topics in uncertainty estimation.
\newblock https://www.cis.upenn.edu/~aaroth/uncertainty-notes.pdf, 2022.

\bibitem[Schapire(1990)]{schapire1990strength}
Robert~E Schapire.
\newblock The strength of weak learnability.
\newblock \emph{Machine learning}, 5\penalty0 (2):\penalty0 197--227, 1990.

\bibitem[Schapire and Freund(2013)]{schapire2013boosting}
Robert~E Schapire and Yoav Freund.
\newblock Boosting: Foundations and algorithms.
\newblock \emph{Kybernetes}, 2013.

\bibitem[Shabat et~al.(2020)Shabat, Cohen, and Mansour]{shabat2020sample}
Eliran Shabat, Lee Cohen, and Yishay Mansour.
\newblock Sample complexity of uniform convergence for multicalibration.
\newblock \emph{Advances in Neural Information Processing Systems},
  33:\penalty0 13331--13340, 2020.

\bibitem[Shalev-Shwartz and Ben-David(2014)]{shalev2014understanding}
Shai Shalev-Shwartz and Shai Ben-David.
\newblock \emph{Understanding machine learning: From theory to algorithms}.
\newblock Cambridge university press, 2014.

\bibitem[Vapnik and Chervonenkis(1971)]{vcdim}
V.N. Vapnik and A.~YA. Chervonenkis.
\newblock On the uniform convergence of relative frequencies of events to their
  probabilities, 1971.

\end{thebibliography}

\ifarxiv
\appendix
\section{Generalization Bounds}
\label{sec:generalization}
Our analysis of Algorithm \ref{alg:regression-multicalibrator} assumed direct access to the data distribution $\cD$. In practice, we will run the algorithm on the empirical distribution over a sample of $n$ points $D \sim \cD^n$. In this section, we show that when we do this, so long as $n$ is sufficiently large, both our squared error and our multicalibration guarantees carry over from the empirical distribution over $D$ to the distribution $\cD$ from which $D$ was sampled. Most generalization bounds for multicalibration algorithms (e.g. \cite{hebert2018multicalibration,jung2021moment,jung2022batch,shabat2020sample}) are either stated and proven for finite classes $\cH$, or are proven for algorithms that do not operate as empirical risk minimization algorithms, but instead gain access to a fresh sample of data from the distribution at each iteration, or are proven for hypotheses classes that are fixed independently of the algorithm. We have a different challenge: Like \cite{hebert2018multicalibration,jung2021moment} we study an iterative algorithm whose final hypothesis class is not fixed up front, but implicitly defined as a function of $\cH$. But we wish to study the algorithms as they are used---as empirical risk minimization algorithms---so we do not want our analysis to depend on using a fresh sample of data at each iteration. And unlike the analysis in \cite{jung2022batch}, for us $\cH$ is continuously large (since it is closed under affine transformations), so we cannot rely on bounds that depend on $\log |\cH|$. Instead we give a uniform convergence analysis that depends on the pseudo-dimension of our class of weak learners $\cH$:

\begin{definition}{Pseudodimension}[\cite{pollard2012convergence}]
Let $\cH$ be a class of functions from $\cX$ to $\R$. We say that a set $S = (x_1, \ldots, x_m, y_1, \ldots, y_m) \in \cX^m \times \R^m$ is pseudo-shattered by $\cH$ if for any $(b_1, \ldots, b_m) \in \{0,1\}^m$ there exists $h \in \cH$ such that $\forall i, h(x_i) > y \Longleftrightarrow b_i = 1$ The pseudodimension of $\cH,$ denoted $\pdim(\cH)$ is the largest integer $m$ for which $\cH$ pseudo-shatters some set $S$ of cardinality $m$.
\end{definition}

Although hypotheses in $\cH$ are continuously valued, Algorithm \ref{alg:regression-multicalibrator} outputs functions that have finite range $[1/m]$, and so we can view them as multi-class classification functions. Our analysis will proceed by studying the generalization properties of these multiclass functions, which we will characterize using Natarajan dimension:

\begin{definition}[Shattering for multiclass functions]\cite{natarajan1989learning, shalev2014understanding}
A set $C \subseteq \cX$ is shattered by $\cH$ if there exists two functions $f_0, f_1:C \rightarrow [k]$ such that 
\begin{enumerate}
\item For every $x \in C, f_0(x) \ne f_1(x)$.
\item For every $B \subseteq C$ there exists a function $h \in \cH$ such that 
\[
\forall x \in B, h(x) = f_0(x) \text{ and } \forall x \in C \ B, h(x)=f_1(x).
\]
\end{enumerate}

\end{definition}

\begin{definition}[Natarajan dimension]\cite{natarajan1989learning, shalev2014understanding}
The Natarajan dimension of $\cH$, denoted $\ndim(\cH),$ is the maximal size of a shattered set $C \subseteq \cX$.
\end{definition}

We can then rely the following standard uniform convergence bound for multiclass classification. This statement is slightly modified from the result in Shalev-Schwartz and Ben-David to account for our use of squared error. The result still holds on account of the fact that the Cherhoff bound only relies on the loss function being bounded, and ours is indeed bounded between 0 and 1. 

\begin{theorem}[Multiclass uniform convergence]\cite{shalev2014understanding}
\label{thm:uc}
Let $\epsilon, \delta > 0$ and let $\cH$ be a class of functions $h: \cX \rightarrow [1/k]$ such that the Natarajan dimension of $\cH$ is $d$. Let $\cD \in \Delta (\cX \times [0,1]) $ be an arbitrary distribution and let $D = \{(x_1, y_1), \ldots, (x_n, y_n)\}_{(x_i, y_i) \sim \cD}$ be a sample of $n$ points from $\cD$.  Then for 
\[ n = O\left( \frac{d \log(k) + \log(1/\delta)}{\eps^2}\right),\]
\[
\Pr\left[\max_{h \in \cH}\left\vert \E_{(x,y)\sim \cD} [(y-h(x))^2] - \E_{(x,y) \sim D} [(y-h(x))^2] \right\vert \ge \epsilon]\right ] \le \delta. 
\]
\end{theorem}

Our strategy will be to bound the Natarajan dimension of the class of models that can be output by Algorithm \ref{alg:regression-multicalibrator} in terms of the pseudodimension of the underlying weak learner, then apply the above uniform convergence result. To do so, we will first use the following lemma, which bounds the Natarajan dimension of functions that can be described as post-processings of binary valued-functions from a class of bounded VC-dimension. 

\begin{lemma}\cite{shalev2014understanding}
\label{lem:ndim-bound}
Suppose we have $\ell$ binary classifiers from binary class $\cH_{\textup{bin}}$ and a rule $r: \{0,1\}^\ell \rightarrow [k]$ that determines a multiclass label according to the predictions of the $\ell$ binary classifiers. Define the hypothesis class corresponding to this rule as 
\[
\cH = \{ r(h_1(\cdot), \ldots, h_\ell(\cdot)) : (h_1, \ldots, h_\ell) \in (\cH_\text{bin})^\ell\}.
\]
Then, if $d = \textup{VCdim}(\cH_\textup{bin}),$
\[
\ndim(\cH) \le 3 \ell d \log (\ell d).
\]
\end{lemma}

Recall that the VC-dimension of a binary classifier is defined as follows:

\begin{definition}[VC-dimension]\cite{vcdim} Let $\cH$ be a class of binary classifiers $h: \cX \rightarrow \{0,1\}$. Let $S = \{x_1, \ldots, x_m\}$ and let $\Pi_\cH(S) = \{ \left( h(x_1), \ldots, h(x_m) \right): h \in \cH\} \subseteq \{0,1\}^m$. We say that $S$ is shattered by $\cH$ if $\Pi_\cH(S) = \{0,1\}^m$. The Vapnik-Chervonenkis (VC) dimension of $\cH$, denoted $\text{VCdim}(\cH)$, is the cardinality of the largest set $S$ shattered by $\cH$.

\end{definition}

\begin{lemma}
\label{lem:final-bound}
Let $\hboost$ be the class of models output by $\textup{RegressionMulticalibrate}(f, \alpha, A_\cH, \cdot, B)$ (Algorithm \ref{alg:regression-multicalibrator}) for any input distribution $\cD$ and let $d$ be the pseudodimension of its input weak learner class $\cH$. 
\[ 
\ndim\left({\hboost}\right)\le 24(B/\alpha)^3 d \log \left((2B/\alpha)^3 d\right).
\]
\end{lemma}

\begin{proof}

Let $m$ be defined (as in $RegressionMulticalibrate(f, \alpha, A_\cH, \cD, B)$) to be $2B/\alpha$. Because our models are always rounded to the nearest value in $[1/m]$, we can think of the model $f_t$ generated in every round of the algorithm multiclass classification problems over $m$ classes. We will show that our final model can be written as a decision rule that maps the outputs of some $\ell$ Boolean classifiers  to $[1/m]$, and that these Boolean classifiers have VC dimension that is bounded by the pseudodimension of the weak learner class. Then, we will apply Lemma \ref{lem:ndim-bound} to get an upper bound on the Natarajan dimension of the class of models in terms of $\alpha, B$, and the pseudodimension of the input weak learner class $\cH$.

Consider the initial round of the algorithm. We can convert our (rounded) initial regressor $f_0$ to a series of $m$ Boolean thresholdings $g_v$ which return 1 when $f_0(x) \ge v$:

\[
g_{v}^{0} = 
\begin{cases}
    1 &\text{if } f_0(x) \ge v,\\ 
    0 &\text{otherwise.}
    \end{cases}.
\]

These $m$ Boolean thresholdings can then be mapped back to the original prediction over $[1/m]$ using a decision rule $r: \{0,1\}^m \rightarrow [1/m]$ which picks the largest of the thresholds that evaluates to 1, and assigns that index to the prediction:

\[
r_0(\{g_v^0\}_{v \in [1/m]})(x) = \arg\max_{i \in [1/m]} i \ind[g_v(x) = 1].
\]

Note that since our initial predictor $f_0$ was already rounded to take values in $[1/m]$, the largest $v$ such that $f_0(x) \ge v$ will be exactly $f_0(x)$, so $r_0$ is exactly equivalent to $f_0$. Similarly, at round $t+1$ of $\textup{RegressionMulticalibrate}(f, \alpha, A_\cH, \cD, B)$, we will show that the model $f_{t+1}$ can be written as a decision rule $r_{t+1}$ over $m + (t+1) m^2$ binary classifiers $g$, where 
\[
g_{v,i}^{t} = 
    \begin{cases}
    1 &\text{if } h_v^{t}(x) \ge i - 1/(2m),\\ 
    0 &\text{otherwise.}
    \end{cases},
\] 

Here, the thresholds measure halfway between each level set, as $h_v^{t}(x)$ has yet to be rounded to the nearest level set. We can write a decision rule that maps these thresholds to classifications over $[1/m]$:
\[
r_{t+1}\left(r_t, \{g_{v,i}^{t+1}\}_{i,v \in [1/m]}\right)(x) = \sum_{v \in [1/m]} \ind[r_t(x)=v] \arg\max_{i \in [1/m]} \left( i \cdot \ind[g_{v,i}^{t+1}(x) = 1] \right),
\]

Now, we need to show that this decision rule evaluated at round $t$ is equivalent to $f_t$. We proceed inductively. For our base case, we have already argued that our initial decision rule $r_0$ is equivalent to the classifier $f_0$. Now, say that we have decision rule $r_t$ over binary classifiers $g$ that is equivalent to model $f_t$. Then, we can write 

\begin{align*}
r_{t+1}\left(r_t, \{g_{v,i}^{t+1}\}_{i,v \in [1/m]}\right)(x) &= \sum_{v \in [1/m]} \ind[r_t(x)=v] \arg\max_{i \in [1/m]} \left( i \cdot \ind[g_{v,i}^{t+1}(x) = 1] \right), \\
&= \sum_{v \in [1/m]} \ind[f_t(x) = v] \arg\max_{i \in [1/m]} \left( i \cdot \ind[g_{v,i}^{t+1}(x) = 1] \right) \\
&= \sum_{v \in [1/m]} \ind[f_t(x) = v] \arg\max_{i \in [1/m]} \left( i \cdot \ind[h_v^{t+1}(x) \ge i - 1/(2m)] \right) \\
& = \sum_{v \in [1/m]} \ind[f_t(x) = v] \round(h_v^{t+1}(x)) \\
& = f_{t+1}(x),
\end{align*}

\noindent where the second line comes from the inductive hypothesis and the second to last line's equality comes from the fact that the largest $i$ such that $h_v^{t+1}(x) - 1/(2m) \ge i$ will be the exact rounded prediction of $h_v^{t+1}(x)$.

Now, we need to show that at round $t+1$, the decision rule is a decision rule over $m + (t+1)m^2$ binary classifiers. Note that our initial decision rule $r_0$ has $m = m + 0 \cdot m^2$ binary classifiers. Say that at round $t$ we have a decision rule $r_t$ over $m + tm^2$ classifiers. In the following round, we build $m^2$ new Boolean classifiers $g_v,i$ for $v,i \in [1/m]$. So, at round $t+1$ we have $m + tm^2 + m^2 = m + (t+1)m^2$ classifiers total. 

From Theorem \ref{thm:alg-analysis}, we know that Algorithm \ref{alg:regression-multicalibrator} halts after at most $T \le 2B/\alpha$ rounds, at which point it outputs model $f_{T-1}$. So, we can rewrite $f_{T-1}$ as a decision rule $r_{T-1}$ composed of at most $m + (T-1)m^2 < Tm^2$ Boolean models. Plugging in our bound for $T$ and definition of $m$, this gives us a decision rule $r_{T-1}$ composed of at most $\left(\frac{2B}{\alpha}\right)^3$ Boolean classifiers. 

Let $\cG$ be the class of Boolean threshold functions over $\cH$, i.e. functions $g: \cX \rightarrow \{0,1\}$ such that 

\[
g(x) = 
\begin{cases}
1 & h(x) \ge i \\
0 & h(x) < i,
\end{cases}
\]

\noindent for some $h \in \cH$ and $i \in \R$. Say that the VC-dimension of $\cG$ is $d'$.
Then, applying lemma \ref{lem:ndim-bound}, it follows that 

\begin{align*}
\ndim(\hboost) &\le 3 \left(\frac{2B}{\alpha}\right)^3 d' \log \left(\left(\frac{2B}{\alpha}\right)^3 d'\right), \\
&= 24 \left(\frac{B}{\alpha}\right) d' \log \left( \left(\frac{2B}{\alpha}\right)^3 d' \right).
\end{align*}

Now, it remains to show that we can bound the VC-dimension of these thresholding functions by the pseudodimension of the weak learner class $\cH$. Note that $\cG$ as we have defined it above is a richer hypothesis class than the actual class of thresholding functions used in the above analysis, because it can threshold at any value in $\R$ rather than being restricted to $[1/m]$. Thus, its VC dimension can only be greater than the VC dimension of the class of threshold functions over $\cH$ restricted to $[1/m]$, and hence an upper bound on the VC dimension of $\cG$ in terms of the pseudodimension of $\cH$ will also be an upper bound on the VC dimension of the restricted class of threshold functions.

Let $d$ be the pseudodimension of $\cH$, and say that $d < d'$. By the definition of VC-dimension, $\{0,1\}^{d+1}$ must be shattered by $\cG$. I.e., for any set of $d+1$ points $x_1, \ldots, x_{d+1} \in \cX$ with arbitrary labels $b_1, \ldots, b_{d+1}$, there is some hypothesis $g \in \cG$ that realizes those labels on $(x_1, \ldots, x_{d+1})$. Consider the function $g$ that, given the $d+1$ points in $\cX$, realizes the labels $b_1, \ldots, b_{d+1}$. By the construction of $\cG$, $g$ is a thresholding of some function $h \in \cH$ at some point $i$. So, there is be some $i \in \R$ such that $h(x_i) > i \Rightarrow b_i = 1$ and such that $b_i = 1 \Rightarrow h(x_i) > i$. But this means that $\{0,1\}^{d+1}$ is pseudo-shattered by $\cH,$ and thus the pseudodimension of $\cH$ is not $d$. Thus, it cannot be the case that $d < d',$ and hence $d' \le d$, i.e. the VC dimension of $\cG$ is bounded above by the pseudodimension of $\cH$. Plugging this bound into the above bound on the Natarajan dimension gives us that

\begin{align*}
\ndim(\hboost) &\le 24 \left(\frac{B}{\alpha}\right) d' \log \left( \left(\frac{2B}{\alpha}\right)^3 d' \right),\\
&\le 24 \left(\frac{B}{\alpha}\right) d \log \left( \left(\frac{2B}{\alpha}\right)^3 d \right).
\end{align*}
\end{proof}

Now, we can state the following uniform convergence theorem for our final model. 
\begin{theorem}[Squared Error Generalization for Algorithm \ref{alg:regression-multicalibrator}.]
\label{thm:alg-convergence}
Let $\epsilon, \delta, \alpha, B > 0$. Let $\hboost$ be the class of models that can be output by $\textup{RegressionMulticalibrate}(f, \alpha, A_\cH, \cdot, B)$ (Algorithm \ref{alg:regression-multicalibrator}) for any input distribution $\cD$ and let $d$ be the pseudodimension of its input weak learner class $\cH$. Let $D = \{(x_1, y_1), \ldots, (x_n, y_n)\}_{(x_i, y_i) \sim \cD}$ be a sample of $n$ points drawn i.i.d. from $\cD$.  Then if 
\[ n = O\left(\frac{dB^3\log^2(dB/\alpha)}{\alpha^3\epsilon^2} + \frac{\log(1/\delta)}{\epsilon^2} \right)\] 
\[
\Pr\left[\max_{h \in \hboost}\left\vert \E_{(x,y)\sim \cD} [(y-h(x))^2] - \E_{(x,y) \sim D} [(y-h(x))^2] \right\vert \ge \epsilon]\right ] \le \delta. 
\]
\end{theorem}

\begin{proof}
This follows directly from Theorem \ref{thm:uc} and the bound on the Natarajan dimension in Lemma \ref{lem:final-bound}.
\end{proof}

We also would like to know that our multicalibration guarantees are generalizable. Rather than doing a bespoke analysis here, we can rely on the connection that we have established between failure of multicalibration and ability to improve squared error and argue that if the final hypothesis output by the algorithm was not multicalibrated with high probability then it would be possible to improve its squared error out-of-sample. Thus, by our previous generalization result for squared error, it would be possible to improve the squared error in-sample as well, giving us a contradiction.

\begin{theorem}[Multicalibration generalization guarantee]
Let $\epsilon, \delta, \alpha, B > 0$ and consider the model $f_{T-1}$ output by $\textup{RegressionMulticalibrate}(f, \alpha, A_\cH, D, B)$ for some sample $D$ of $n$ points drawn i.i.d. from distribution $\cD$ such that 
\[ n = O\left(\frac{dB^3\log^2(dB/\alpha)}{\alpha^3\epsilon^2} + \frac{\log(1/\delta)}{\epsilon^2} \right)\] 
Then if $\epsilon \leq \frac{\alpha}{4B}$, with probability greater than or equal to $1-2\delta$ it follows that $f_{T-1}$ is $2\alpha$-approximately multicalibrated with respect to the distribution $\cD$.
\end{theorem}

\begin{proof}
Let $D = \{(x_1,y_1), \ldots, (x_n, y_n)\}_{(x_i, y_i) \sim \cD}$. Consider the model $f_{T-1}$ output by \\ $\textup{RegressionMulticalibrate}(f, \alpha, A_\cH, D, B)$, and recall that within the run of the algorithm there was also a model $f_T$ defined in the final round. Say that model $f_{T-1}$ is not $2\alpha$-approximately multicalibrated with respect to $\cH_B$ and the true distribution $\cD$.

Since the algorithm running on the sample halted, it must have been that the model in the final round improved in squared error by less than $\alpha/(2B)$ when measured with respect to the sample $D$:

\[\E_{(x,y)\sim D}[(f_{T-1} - y)^2] - \E_{(x,y)\sim D}[(f_{T} - y)^2] \le (\alpha/2B).\]

Consider what happens if we run the algorithm again, but with $f_{T-1}$ as its initial model and now with the underlying distribution as input rather than the sample of $n$ points. Let $f'_T$ be the model found in the first round of running this process $\textup{RegressionMulticalibrate}(f_{T-1}, \alpha, A_\cH, \cD, B)$. Since $f_{T-1}$ is not $2\alpha-$approximately multicalibrated with respect to $\cD$ and $\cH_B$, then by an identical argument as in the proof of Theorem \ref{thm:alg-analysis}, it it must be that a single round of the algorithm improves the squared error on $\cD$ by at least $\alpha/B$. Thus, $\E_{(x,y)\sim \cD}[(f_{T-1} - y)^2] - \E_{(x,y)\sim \cD}[(f'_{T} - y)^2] >  \alpha/B$.

We know from our previous convergence bound, Theorem \ref{thm:alg-convergence}, that with probability $1- \delta, \vert \E_{(x,y)\sim D}[(f'_T - y)^2] - \E_{(x,y)\sim \cD}[(f_T - y)^2] \vert < \epsilon$. So, $f'_T$ must with high probability also improve on the sample $D$:

\begin{align*}
\frac{\alpha}{B} &< \E_{(x,y)\sim \cD}[(f_{T-1} - y)^2] - \E_{(x,y)\sim \cD}[(f'_{T} - y)^2] \\
& < \E_{(x,y)\sim \cD}[(f_{T-1} - y)^2] - \E_{(x,y)\sim D}[(f'_{T} - y)^2] + \epsilon &\text{(with probability $\ge 1 - \delta$) } \\
& < \E_{(x,y)\sim D}[(f_{T-1} - y)^2] - \E_{(x,y)\sim D}[(f'_{T} - y)^2] + 2\epsilon &\text{(with probability $\ge 1 - 2\delta$) } \\
& < \frac{\alpha}{2B} + 2\epsilon,
\end{align*}

where the last line comes from the fact that the error of $f'_{T}$ on $D$ cannot be less than the error of $f_T$ on $D$, or else the regression oracle would have found it. Now we have a contradiction: since we have set $\epsilon \le \frac{\alpha}{4B}$, 

\begin{align*}
\frac{\alpha}{B} &< \frac{\alpha}{2B} + 2\frac{\alpha}{4B} \\
    &= \frac{\alpha}{B}.
\end{align*}

So, it must follow that $f_{T-1}$ is, with probability $1-2\delta$, $2\alpha-$approximately multicalibrated. 
\end{proof}


\fi

\end{document}